\documentclass{article}

\usepackage{microtype}
\usepackage{graphicx}
\usepackage{booktabs} 

\usepackage{hyperref}

\usepackage[accepted]{icml2020}

\usepackage{amsmath}
\usepackage{amssymb}
\usepackage{amsthm}
\usepackage{subcaption}

\usepackage{amsmath,amsfonts,bm}

\def\Figref#1{Figure~\ref{#1}}

\def\eqref#1{equation~\ref{#1}}
\def\Eqref#1{Equation~\ref{#1}}

\def\1{\bm{1}}

\def\rva{{\mathbf{a}}}
\def\rvb{{\mathbf{b}}}

\def\rvd{{\mathbf{d}}}
\def\rve{{\mathbf{e}}}

\def\rvu{{\mathbf{i}}}

\def\rvu{{\mathbf{u}}}
\def\rvv{{\mathbf{v}}}
\def\rvw{{\mathbf{w}}}
\def\rvx{{\mathbf{x}}}
\def\rvy{{\mathbf{y}}}
\def\rvz{{\mathbf{z}}}

\def\trvd{{\tilde{\rvd}}}
\def\trva{{\tilde{\rva}}}

\def\rvzero{{\mathbf{0}}}

\def\rmA{{\mathbf{A}}}
\def\rmB{{\mathbf{B}}}

\def\rmD{{\mathbf{D}}}
\def\rmE{{\mathbf{E}}}

\def\rmG{{\mathbf{G}}}

\def\rmI{{\mathbf{I}}}

\def\rmM{{\mathbf{M}}}

\def\rmP{{\mathbf{P}}}

\def\rmW{{\mathbf{W}}}

\def\trmD{{\tilde{\rmD}}}
\def\trmW{{\tilde{\rmW}}}
\def\trmG{{\tilde{\rmG}}}
\def\trmA{{\tilde{\rmA}}}

\DeclareMathAlphabet{\mathsfit}{\encodingdefault}{\sfdefault}{m}{sl}
\SetMathAlphabet{\mathsfit}{bold}{\encodingdefault}{\sfdefault}{bx}{n}

\def\gF{{\mathcal{F}}}

\def\gN{{\mathcal{N}}}

\def\gS{{\mathcal{S}}}

\def\emC{{C}}
\def\emD{{D}}
\def\emE{{E}}

\def\emG{{G}}

\def\emT{{T}}

\def\emW{{W}}

\newcommand{\E}{\mathbb{E}}

\newcommand{\R}{\mathbb{R}}

\newcommand{\Var}{\mathrm{Var}}

\newcommand{\normlzero}{L^0}
\newcommand{\normlone}{L^1}

\DeclareMathOperator*{\argmin}{arg\,min}

\DeclareMathOperator*{\prox}{prox}
\newcommand{\st}{\operatorname{s.t.}}

\DeclareMathOperator*{\minimize}{\text{minimize}}

\newcommand{\norm}[1]{\left\lVert#1\right\rVert}
\newcommand{\norminf}[1]{\left\lVert#1\right\rVert_{\infty}}
\newcommand{\abs}[1]{\left\lvert#1\right\rvert}

\newtheorem{theorem}{Theorem}

\newtheorem{definition}{Definition}

\newcommand{\thmax}{\theta_{\max}}
\newcommand{\thmin}{\theta_{\min}}
\newcommand{\Supp}{{\textnormal{Supp}}}
\def\tmu{{\widetilde \mu}}
\newcommand{\thickbar}[1]{\mathbf{\bar{\text{$#1$}}}}
\def\bmu{{\thickbar \mu}}
\def\taud{{\tau_{\textnormal{d}}}}
\def\tauod{{\tau_{\textnormal{od}}}}
\def\Vd{{v_{\textnormal{d}}}}
\def\Vod{{v_{\textnormal{od}}}}
\def\wdd{{w_{\textnormal{d}}}}
\def\brmG{{\thickbar \rmG}}
\def\bemG{{\thickbar \emG}}
\def\etrmG{{\tilde{\emG}}}

\newcommand{\inpaintingl}[1]{\includegraphics[width=0.22\linewidth]{figures/inpainting/ratio_0_5_lambd_0_1_T20/#1}}
\newcommand{\inpaintings}[1]{\includegraphics[width=0.20\linewidth]{figures/inpainting/ratio_0_5_lambd_0_1_T20/#1}}

\icmltitlerunning{Ada-LISTA: Learned Solvers Adaptive to Varying Models}

\begin{document}

\twocolumn[
\icmltitle{Ada-LISTA: Learned Solvers Adaptive to Varying Models}




\begin{icmlauthorlist}
\icmlauthor{Aviad Aberdam}{EE}
\icmlauthor{Alona Golts}{CS}
\icmlauthor{Michael Elad}{CS}
\end{icmlauthorlist}

\icmlaffiliation{EE}{Department of Electrical Engineering, Technion Institute of Technology, Israel.}
\icmlaffiliation{CS}{Department of Computer Science, Technion Institute of Technology, Israel}

\icmlcorrespondingauthor{Aviad Aberdam}{aaberdam@cs.technion.ac.il}
\icmlcorrespondingauthor{Alona Golts}{salonaz@cs.technion.ac.il}
\icmlcorrespondingauthor{Michael Elad}{elad@cs.technion.ac.il}

\icmlkeywords{Machine Learning, ICML, Sparse coding, Optimization, Learned solvers, Theory of neural networks}

\vskip 0.3in
]



\printAffiliationsAndNotice{}  

\begin{abstract}
Neural networks that are based on unfolding of an iterative solver, such as LISTA (learned iterative soft threshold algorithm), are widely used due to their accelerated performance. 
Nevertheless, as opposed to non-learned solvers, these networks are trained on a certain dictionary, and therefore they are inapplicable for varying model scenarios. 
This work introduces an adaptive learned solver, termed Ada-LISTA, which receives pairs of signals and their corresponding dictionaries as inputs, and learns a universal architecture to serve them all.
We prove that this scheme is guaranteed to solve sparse coding in linear rate for varying models, including dictionary perturbations and permutations. We also provide an extensive numerical study demonstrating its practical adaptation capabilities. Finally, we deploy Ada-LISTA to natural image inpainting, where the patch-masks vary spatially, thus requiring such an adaptation. 
\end{abstract}


\section{Introduction}

Sparse coding is the task of representing a noisy signal $\rvy \in \R^n$ as a combination of few base signals (called ``atoms''), taken from a matrix $\rmD \in \R^{n \times m}$ -- the ``dictionary''. This is represented as the need to compute $\rvx \in \R^m$ such that  
\begin{equation}
    \rvy \approx \rmD \rvx, \quad \st ~ \|\rvx\|_0 \leq s,
\end{equation}
where the $\normlzero$-norm counts the non-zero elements, $s$ is the cardinality of the representation, and $\rmD$ is often redundant ($m \geq n$). Among the various approximation methods for handling this NP-hard task, an appealing approach is a relaxation of the $\normlzero$ to an $\normlone$-norm using Lasso or Basis-Pursuit \cite{LASSO,chen2001atomic}, 
\begin{equation} \label{eq:basis_pursuit_objective}
    \underset{\rvx}{\text{minimize}} \frac{1}{2} \|\rvy - \rmD \rvx\|_2^2 + \lambda \|\rvx\|_1.
\end{equation}
An effective way to address this optimization problem uses an iterative algorithm such as ISTA (Iterative Soft Thresholding Algorithm) \cite{ista}, where the solution is obtained by iterations of the form
\begin{equation}\label{eq:ista_step}
    \rvx_{k+1} = \gS_{\frac{\lambda}{L}}\left(\rvx_k + \frac{1}{L}\rmD^T(\rvy - \rmD \rvx_k)\right), k=0,1, ..
\end{equation}
where $\frac{1}{L}$ is the step size determined by the largest eigenvalue of the Gram matrix $\rmD^T \rmD$, and $\gS_{\theta}(\rvx_i) = \text{sign}(\rvx_i)(|\rvx_i|-\theta_i)$ is the soft shrinkage function. Fast-ISTA (FISTA) \cite{beck2009fast} is a speed-up of the above iterative algorithm, which should remind the reader of the momentum method in optimization.

\begin{figure}[t]
    \centering
    \includegraphics[width=0.7\linewidth]{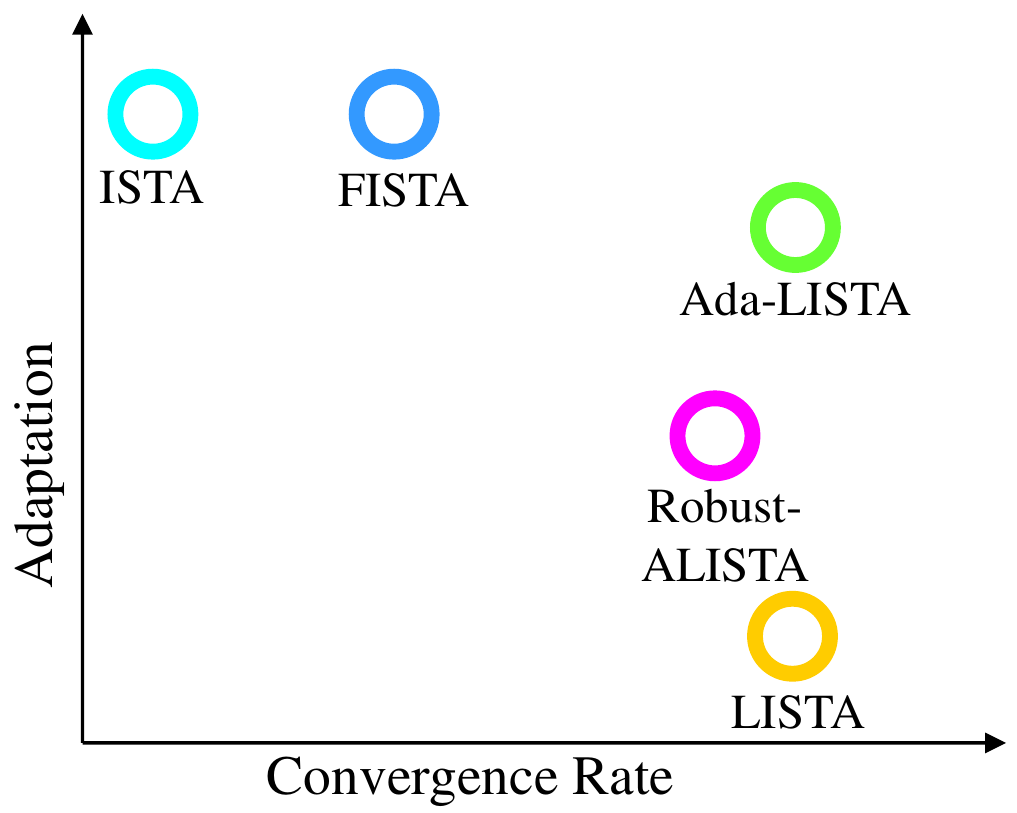}
    \caption{We propose ``Ada-LISTA'', a fusion between the flexible ISTA and FISTA schemes, receiving both the signal and dictionary at inference time, and the highly efficient learned solvers, LISTA and ALISTA.}
    \label{fig:rate_adaptation}
\end{figure}

As a side note, we mention that ISTA has a much wider perspective when aiming to minimize a function of the form
\begin{equation}\label{eq:convex_problem}
    F(\rvx) = f(\rvx) + g(\rvx),
\end{equation}
where $f$ and $g$ are convex functions, with $g$ possibly non-smooth. The solution is given by the proximal gradient method \cite{combettes2005signal,beck2017first}:
\begin{equation} \label{eq:proximal_method}
\begin{split}
    & \rvx_{k+1} = \prox_{g}\left( \rvx_k - \frac{1}{L}\nabla f(\rvx_k)  \right), \\
    & \prox_{g}(\rvu) = \argmin_{\rvv} \frac{1}{2}\norm{\rvv-\rvu}_2^2 + g(\rvv).
\end{split}
\end{equation}
The above  fits various optimization problems such as a projected gradient descent over an indicator function $g$, the matrix completion problem \cite{mazumder2010spectral}, portfolio optimization \cite{boyd2004convex}, non-negative matrix factorization \cite{sprechmann2015learning}, and more.

Returning to the realm of sparse coding, the seminal work of LISTA (Learned-ISTA) \cite{LISTA} has shown that by unfolding $K$ iterations of ISTA and freeing its parameters to be learned, one can achieve a substantial speedup over ISTA (and FISTA). Particularly, LISTA uses the following re-parametrization: 
\begin{equation}
    \rvx_{k+1} = \gS_{\theta} (\rmW_1 \rvy + \rmW_2 \rvx_k), \quad k=0,1,...,K-1,
\end{equation}
where $\rmW_1$ and $\rmW_2$ re-parametrize the matrices $\frac{1}{L} \rmD^T$ and $(\rmI - \frac{1}{L} \rmD^T \rmD)$ correspondingly. These two matrices and  the scalar thresholding value $\theta$ are collectively referred to as $\rm\Theta = (\rmW_1,\rmW_2,\theta)$ -- the parameters to be learned. The model, denoted as $\gF_K(\rvy; \rm\Theta)$, is trained by minimizing the squared error between the predicted sparse representations at the $K$th unfolding $\rvx_K=\gF_K(\rvy; \rm\Theta)$, and the optimal codes $\rvx$ obtained by running ISTA itself,
\begin{equation}\label{LISTA-loss}
    \minimize_{\rm\Theta} \sum_{i=1}^N \left\|\gF_K(\rvy_i; \rm\Theta) - \rvx_i \right\|_2^2.
\end{equation}

Once trained, LISTA requires only the test signals during inference, without their underlying dictionary. It has been shown in \cite{LISTA} that LISTA generalizes well for signals of the same distribution as in the train set, allowing a significant speedup versus its non-learned counterparts. This may be explained by the fact that while non-learned solvers do not make any assumption on the input signals, LISTA fits itself to the input distribution. More specifically, in sparse coding, the input signals are restricted to a union of low-dimensional Gaussians, as they are generated by a linear combination of few atoms. By focusing on such signals solely, this allows LISTA to achieve its acceleration. Note, however, that the original dictionary is hard-coded into the model weights via the ground truth solutions used during the supervised training. Given a new test sample that emerges from a slightly deviated (yet known) model/dictionary, LISTA will most likely deteriorate in performance, whereas ISTA and FISTA are expected 
to provide a robust and consistent result, as they are agnostic to the input signals and dictionary.

From a different point of view, a drawback of LISTA is its relevance to a single dictionary, requiring a separate and renewed training if the model evolves over time. Such is the case in video related applications as enhancement \cite{video1} or surveillance \cite{video2}, where the dictionary should vary along time. Similarly, in some image restoration problems, the model encapsulated by the dictionary is often corrupted by an additional constant perturbation, e.g., the sensing matrix in compressive sensing \cite{compressive_sensing}, the blur kernel in non-blind image deblurring \cite{image_deblurring}, and a spatially-varying mask in image inpainting \cite{image_inpainting}. In all these cases, deployment of the classic framework of LISTA necessitates a newly trained network for each new dictionary. An alternative to the above is incorporating LISTA as a fixed black-box denoiser, and merging it within the plug-and-play \cite{venkatakrishnan2013plug} or RED \cite{romano2017little} schemes, significantly increasing the inference complexity. 

\paragraph{Main Contributions:}
Our aim in this work is to extend the applicability of LISTA to scenarios of model perturbations and varying signal distributions. More specifically, 
\begin{itemize}
    \item We bridge the gap between the efficiency and the fast convergence rate of LISTA, and the high adaptivity and applicability of ISTA (and FISTA), by introducing ``Ada-LISTA'' (Adaptive-LISTA). Our training is based on pairs of signals and their corresponding dictionaries, learning a generic architecture that wraps the dictionary by two auxiliary weight matrices. At inference, our model can accommodate the signal and its corresponding dictionary, allowing to handle a variety of model modifications without repetitive re-training.
    \item We perform extensive numerical experiments, demonstrating the robustness of our model to three types of dictionary perturbations: permuted columns, additive Gaussian noise, and completely renewed random dictionaries. We demonstrate the ability of Ada-LISTA to handle complex and varying signal models while still providing an impressive advantage over both learned and non-learned solvers.
    \item We prove that our modified scheme achieves a linear convergence rate under a constant dictionary. More importantly, we allow for noisy modifications and random permutations to the dictionary and prove that robustness remains, with an ability to reconstruct the ideal sparse representations with the same linear rate. 
    \item We demonstrate the use of our approach on natural image inpainting, which cannot be directly used with hard-coded models as LISTA. We show a clear advantage of Ada-LISTA versus its non-learned counterparts.
    \end{itemize}

\noindent Adopting a wider perspective, our study contributes to the understanding of learned solvers and their ability to accelerate convergence. 
Common belief suggests that the signal model should be structured and fixed for successful learning of such solvers. Our work reveals, however, that effective learning can be achieved with a weaker constraint -- having a fixed conditional distribution of the data given the model $p(\rvy | \rmD)$. 

The LISTA concept of unfolding the iterations of a classical optimization scheme into an RNN-like neural network, and freeing its parameters to be learned over the training data, appears in many works. These include an unsupervised and online training procedure \cite{sprechmann2015learning}, a multi-layer version \cite{multi_layer}, a gated mechanism compensating shrinkage artifacts \cite{wu2020sparse}, as well as reduced-parameter schemes \cite{LISTA_CPSS,liu2018alista}. This paradigm has been brought to various applications, such as compressed sensing, super-resolution, communication, MRI reconstruction \cite{ISTA_net,ldamp,LISTA_superres,amp,ADMM_net,hershey2014deep}, and more. 
A prominent line of work investigates the success of such learned solvers from a theoretical point of view \cite{max_sparsity,l0_encoder,understanding_LISTA,tradeoffs,zarka2019deep}. 
Most of these consider a fixed signal model, with the exception of ``robust-ALISTA'' \cite{liu2018alista} that introduces an adaptive variation of LISTA.
This scheme, however, is restricted to small model perturbations, and cannot address more complicated model variations. 
A more detailed discussion of the relevant literature in relation of our study appears in Section \ref{sec:related_work}.


\section{Proposed Method}
\label{sec:our_method}

\begin{figure*}[ht]
    \centering
    \includegraphics[width=1.0\linewidth]{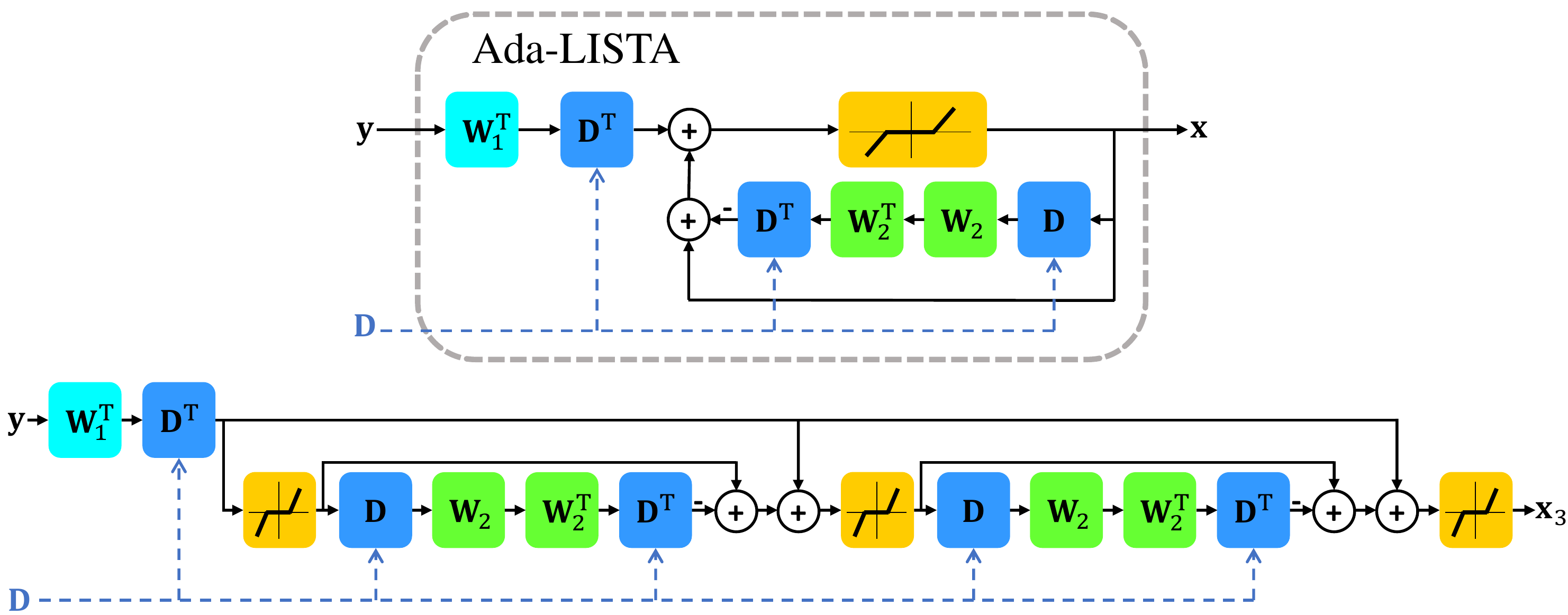}
    \caption{Ada-LISTA architecture as an iterative model (top), and its unfolded version for three iterations (bottom). The input dictionary $\rmD$ is embedded in the architecture, while the matrices $\rmW_1,\rmW_2$ are free to be learned.}
    \label{fig:adalista_scheme}
\end{figure*}

\begin{algorithm}[tbh]
   \caption{Ada-L(F)ISTA Inference}
   \label{alg:adalfista}
\begin{algorithmic}
   \STATE {\bfseries Input:} signal $\rvy$, dictionary $\rmD$
   \STATE {\bfseries Init:} $\rvx_0 = \rvzero$, $\rvz_0 = \rvzero$, $t_0 = 1$
   \FOR{$k=0$ {\bfseries to} $K - 1$}
   \STATE $\rvx_{k+1} = \gS_{\theta_{k+1}} \big( (\rmI - \gamma_{k+1} \rmD^T \rmW_2^T \rmW_2 \rmD) \rvz_k$ \\ 
   \hfill $+ \gamma_{k+1} \rmD^T \rmW_1^T \rvy \big)$
   \IF{Ada-LISTA}
   \STATE $\rvz_{k+1} = \rvx_{k+1}$
   \ELSIF{Ada-LFISTA}
   \STATE $t_{k+1} = \frac{1 + \sqrt{1 + 4 t_k^2}}{2}$
   \STATE $\rvz_{k+1} = \rvx_{k+1} + \frac{t_k - 1}{t_{k+1}}(\rvx_{k+1} - \rvx_k)$
   \ENDIF
   \ENDFOR
   \STATE {\bfseries Return:} $\gF_K(\rvy, \rmD; \rm\Theta) = \rvx_K$
\end{algorithmic}
\end{algorithm}
\begin{algorithm}[tbh]
   \caption{Ada-LISTA Training}
   \label{alg:adalfista_training}
\begin{algorithmic}
   \STATE {\bfseries Input:} pairs of signals and dictionaries $\{\rvy_i, \rmD_i\}_{i=1}^N$
   \STATE {\bfseries Preprocessing:} find $\rvx_i$ for each pair $(\rvy_i, \rmD_i)$ by solving \Eqref{eq:basis_pursuit_objective} using ISTA
   \STATE {\bfseries Goal:} learn $\rm\Theta = (\rmW_1,\rmW_2,\theta_k, \gamma_k)$
   \STATE {\bfseries Init:} $\rmW_1,\rmW_2 = \rmI$, $\theta_k, \gamma_k = 1$
   \FOR{each batch $\{\rvy_i, \rmD_i, \rvx_i\}_{i=1}^{N_B}$}
   \STATE update $\rm\Theta$ by $\partial_{\rm\Theta} \sum_{i \in N_B} \norm{ \gF_K(\rvy_i, \rmD_i; \rm\Theta) - \rvx_i }_2^2 $
   \ENDFOR
\end{algorithmic}
\end{algorithm}

Thus far, as depicted in \Figref{fig:rate_adaptation}, one could either benefit from a high convergence rate using a learned solver as LISTA, while restricting the signals to a specific model, or employ a non-learned and less effective solver as ISTA/FISTA that is capable of handling any pair of signal and its generative model. 
In this paper we introduce a novel architecture, termed ``Adaptive-LISTA'' (Ada-LISTA), combining both benefits. Beyond enjoying the acceleration benefits of learned solvers, we incorporate the dictionary as part of the input at both training and inference time, allowing for adaptivity to different models. \Figref{fig:adalista_scheme} provides our suggested architecture, based on the following: 

\begin{definition}[Ada-LISTA]
\label{def:ada_lista}
    The Ada-LISTA solver is defined\footnote{Although the above definition corresponds to the sparse coding problem, the Ada-LISTA method can be applied to any convex problem formulated as \Eqref{eq:convex_problem}, by swapping the soft-threshold with a different proximal operator (\Eqref{eq:proximal_method}).} by the following iterative step:
    \begin{multline}
        \label{eq:ada_lista_def}
        \rvx_{k+1} = \gS_{\theta_{k+1}} \big( \left( \rmI - \gamma_{k+1} \rmD^T\rmW_1^T\rmW_1\rmD \right) \rvx_{k} \\
        + \gamma_{k+1} \rmD^T\rmW_2^T\rvy \big).
    \end{multline}
    The signal $\rvy$ and the dictionary $\rmD$ are the inputs, and the learned parameters are $\rmW_1,\rmW_2 \in \R^{n \times n}$ and $\{\gamma_k,\theta_k\}$.
\end{definition}

The inference (for both ISTA and FISTA) and the training procedures of Ada-LISTA are detailed in Algorithms \ref{alg:adalfista} and \ref{alg:adalfista_training} correspondingly. We consider a similar loss as in \Eqref{LISTA-loss}, while also incorporating the concurrent dictionaries,
\begin{equation}\label{Ada-LISTA-loss}
    \minimize_{\rm\Theta} \sum_{i=1}^N \left\|\gF_K(\rvy_i, \rmD_i; \rm\Theta) - \rvx_i \right\|_2^2.
\end{equation}
This learning regime is supervised, requiring reference representations $\rvx_i $ to be computed using ISTA. An unsupervised alternative can be envisioned, as in \cite{sprechmann2015learning,golts2018deep}, where the loss is 
\begin{equation*}\label{Ada-LISTA-loss-UN}
    \min_{\rm\Theta}  \sum_{i=1}^N \left\|\rvy_i - \rmD_i \gF_K(\rvy_i, \rmD_i; \rm\Theta) \right\|_2^2
    + \lambda \left\| \gF_K(\rvy_i, \rmD_i; \rm\Theta) \right\|_1.
\end{equation*}
In this paper we shall focus on the supervised mode of learning, leaving the unsupervised alternative to future work. 

Several key questions arise on the applicability of the above learned solver: Does it work? and if so, is performance compromised by Ada-LISTA, as opposed to training LISTA for each separate model? To what extent can it be used? Can it handle completely random models? Can theoretical guarantees be provided on its convergence rate, or adaptation capability? We aim to answer these questions, and we start with a theorem on the robustness of our scheme by proving \emph{linear rate} convergence under \emph{varying} model. 


\section{Ada-LISTA: Theoretical Study}

For the following study, we consider a reduced scheme of Ada-LISTA with a single weight matrix, so as to avoid complication in theorem conditions. We emphasize, however, that the same claims can be derived for our original scheme. 

\begin{definition}[Ada-LISTA -- Single Weight Matrix]
\label{def:adalista_one}
    Ada-LISTA with a single weight matrix is defined by
    \begin{equation}
        \label{eq:ada_lista_one_matrix}
        \rvx_{k+1} = \gS_{\theta_{k+1}} ( \rvx_k + \rmD^T\rmW^T (\rvy-\rmD\rvx_k) ).
    \end{equation}
\end{definition}

We start by recalling the definition of mutual coherence between two matrices:
\begin{definition}[Mutual Coherence]
    Given two matrices, $\rmA$ and $\rmB$, if the diagonal elements of $\rmA^T\rmB$ are equal to $1$, then the mutual coherence is defined as
    \begin{equation}
        \mu(\rmA, \rmB) = \max_{i \neq j} | \rva^T_{i} \rvb_j |,
    \end{equation}
    where $\rva_i$ and $\rvb_j$ are the $i$th and $j$th columns of $\rmA$ and $\rmB$.
\end{definition}

Our first goal is to prove that Ada-LISTA is capable of solving the sparse coding problem in linear rate. We show that if all the signals emerge from the same dictionary $\rmD$, there exists a weight matrix $\rmW$ and threshold values such that the recovery error decreases linearly over iterations. The following theorem indicates that if Ada-LISTA's training reaches its global minimum, the rate would be at least linear. In this part, we follow the steps in \cite{zarka2019deep}, which generalize the proof of ALISTA \cite{liu2018alista} to noisy signals. The proof for Theorem \ref{th:fixed_model} appears in Appendix \ref{ap:proof_1}.

\begin{theorem}[Ada-LISTA Convergence Guarantee]
\label{th:fixed_model}
    Consider a noisy input $\rvy = \rmD \rvx^* + \rve$. If $\rvx^*$ is sufficiently sparse, 
    \begin{equation}
        s = \norm{\rvx^*}_0 < \frac{1}{2\tmu}, \quad \tmu \triangleq \mu(\rmW\rmD, \rmD),
    \end{equation}
    and the thresholds satisfy the condition
    \begin{equation}
        \theta_k =  \thmax\, \gamma^{-k} > \thmin = \frac{\|\rmA^T \rve \|_\infty } {1 - 2 \gamma \tmu s},
    \end{equation}
    with $1 < \gamma < (2 \tmu s)^{-1}$, $\rmA \triangleq \rmW \rmD$, and 
    $\thmax \geq \|\rmA^T \rvy\|_\infty$, then the support in the $k$th iteration of Ada-LISTA (Definition \ref{def:adalista_one}) is included in the support of $\rvx^*$, and its values satisfy
    \begin{equation}
        \|\rvx_k - \rvx^* \|_\infty \leq 2 \,\thmax \,\gamma^{-k}.
    \end{equation}
\end{theorem}

We proceed by claiming that Ada-LISTA can be adaptive to model variations. In this setting, we argue that the signal can originate from different models, and nonetheless there exist global parameters such that Ada-LISTA will converge in linear rate to the original representation. 
Our Theorem exposes the key idea that, as opposed to LISTA which corresponds to a single dictionary, Ada-LISTA can be flexible to various models, while still providing good generalization. Appendix \ref{ap:proof_dictionaries_set} contains the proof of the following Theorem. 

\begin{theorem}[Ada-LISTA -- The Applicable Dictionaries]
\label{th:dictionaries_set}
    Consider a trained Ada-LISTA network with a fixed $\rmW$, and noisy input $\rvy = \rmD \rvx^* + \rve$.
    If the following conditions hold: 

    1. The diagonal elements of $\rmG \triangleq \rmD^T \rmW^T \rmD$ are close to $1$: $\max_i \abs{\emG_{i i} - 1} \leq \epsilon_d$;  
    
    2. The off-diagonals are bounded: $\max_{i \ne j} \abs{\emG_{i j}} \leq \bmu$;
    
    3. $\rvx^*$ is sufficiently sparse: $s = \norm{\rvx^*}_0 < \frac{1}{2 \bmu}$; and

    4. The thresholds satisfy
    \begin{equation}\nonumber
        \theta_k =  \thmax\, \gamma^{-k} > \thmin = \frac{\|\rmA^T \rve \|_\infty } {1 - 2\gamma\epsilon_d - 2 \gamma \bmu s},
    \end{equation}
    with $1 < \gamma < (2 \bmu s)^{-1}$, $\rmA \triangleq \rmW \rmD$, and $\thmax \geq \|\rmA^T \rvy\|_\infty$, 
    
    then the support of the $k$th iteration of Ada-LISTA is included in the support of $\rvx^*$, and its values satisfy
    \begin{equation}
        \|\rvx_k - \rvx^* \|_\infty \leq 2 \,\thmax \,\gamma^{-k}.
    \end{equation}
\end{theorem}

An interesting question arising is the following: Once Ada-LISTA has been trained and the matrix $\rmW$ is fixed, which dictionaries can be effectively served with the same parameters, without additional training? Theorem \ref{th:dictionaries_set} reveals that as long as the effective matrix $\rmG = \rmD^T \rmW^T \rmD$ is sufficiently close to the identity matrix, linear convergence is guaranteed. This holds in particular for two interesting scenarios, proven in Appendices \ref{ap:proof_permutation} and \ref{ap:proof_noisy_dict}:
\begin{enumerate}
    \item \textbf{Random permutations} -- If Ada-LISTA converges for signals emerging from $\rmD$, it also converges for signals originating from any permutation of $\rmD$'s atoms.
    \item \textbf{Noisy dictionaries} -- If Ada-LISTA converges given a clean dictionary $\rmD$, satisfying $\mu(\rmW\rmD, \rmD) < \bmu$, it also converges for noisy models $\trmD = \rmD + \rmE$, with some probability, depending on the distribution of $\rmE$.
\end{enumerate}

To the best of our knowledge, Theorem \ref{th:dictionaries_set} provides the first convergence guarantee in the presence of \emph{model variations}, claiming that linear rate convergence is guaranteed, depending on the availability of small enough cardinality and low mutual-coherence $\tmu$. 
Note that the above claim, as in previous work  \cite{liu2018alista,zarka2019deep}, addresses the core capability of reaching linear convergence rate while disregarding both training and generalization errors.


\section{Related Work}
\label{sec:related_work}

As already mentioned, the literature discussing LISTA and its successors, is abundant. In this section we aim to discuss relevant work to provide better context to our contribution.

The most relevant work to ours is  ``robust-ALISTA'' \cite{liu2018alista}, introducing adaptivity to dictionary perturbations. Their work assumes that every signal $\rvy_i$ comes from a different noisy model $\trmD_i = \rmD + \rmE_i$, where $\rmE_i$ is an interference matrix. For each noisy dictionary $\trmD_i$ this method computes an analytic matrix $\trmW_i$ that minimizes the mutual coherence $\mu(\trmW_i, \trmD_i)$. Then, $\trmW_i$ and $\rmD$ are embedded in the architecture, and the training is performed over the step sizes and the thresholds only, leading to a considerable reduction in the number of trained parameters.

While Robust-ALISTA considers model perturbations only, we show empirically that our method can handle more complicated model deviations, as dictionary permutations and  totally random models. Additionally, in terms of computational complexity, robust-ALISTA has a complicated calculation of the analytic matrices during inference time, a limitation that does not exist in our scheme. We refer the reader to Appendix \ref{app:alista} for a more detailed discussion on the difference between both methods.

As to the theoretical aspect of our study, \cite{LISTA_CPSS,liu2018alista,zarka2019deep,wu2020sparse} have recently shown that learned solvers can achieve linear convergence, under specific conditions on the sparsity level and mutual coherence. These results are the inspiration behind Theorem \ref{th:fixed_model}. This work, however, generalizes these guarantees to a varying model scenario, proving that the same weight matrix can serve different models while still reaching linear convergence.


\section{Numerical Results}

To demonstrate the effectiveness of our approach, we perform extensive numerical experiments, where our goal is two-fold. First we examine Ada-LISTA on a variety of synthetic data scenarios, including column permutations of the input dictionary, additive noisy versions of it, and completely random input dictionaries. Second, we perform a natural image inpainting experiment, showcasing our robustness to a real-world task\footnote{The code for reproducing all experiments is available at \url{https://github.com/aaberdam/AdaLISTA}.}.


\subsection{Synthetic Experiments}
\label{subsec:synthetic}
\paragraph{Experiment Setting.} We construct a dictionary $\rmD\in \R^{50 \times 70}$ with random entries drawn from a normal distribution, and normalize its columns to have a unit $L^2$-norm. Our signals $\rvy \in \R^{50}$ are created as sparse combinations of atoms over this dictionary, $\rvy=\rmD \rvx^*$. While the reported experiments in this section assume no additive noise, Appendix \ref{app:noisy_sig} presents a series of similar tests with varying levels of noise, showing the same qualitative results. The representation vectors $\rvx^* \in \R^{70}$ are created by randomly choosing a support of cardinality $s=4$ with Gaussian coefficients, $\rvx^*_{i \in \text{support}} \in \gN(0,1)$. Instead of using the true sparse representations $\rvx^*$ as ground truth for training, we compute the Lasso solution $\rvx$ with FISTA ($100$ iterations, $\lambda=1$), using the obtained signals $\rvy$ and their corresponding dictionary $\rmD$. This is done in order to maintain a real-world setting, where one does not have access to the true sparse representations. We create in this manner $N=20,000$ examples for training, and $N_{\text{test}}=1,000$ for test. Our metric for comparison between different methods is the MSE (Mean Square Error) between the ground truth $\rvx$ and the predicted sparse representations at $K$ unfoldings, $\|\rvx - \rvx_K\|^2$. In all experiments, the Ada-LISTA weight matrices are both initialized as the identity matrix. In the following set of experiments we gradually diverge from the initial model, given by the dictionary $\rmD$, by applying different degradation/modifications to it.

\paragraph{Random Permutations.}  We start with a scenario in which the columns of the initial dictionary $\rmD$ are permuted randomly to create a new dictionary $\tilde{\rmD}$. This transformation can occur in the non-convex process of dictionary learning, in which different initializations might incur a  different order of the resulting atoms. Although the signals' subspace remains intact, learned solvers as LISTA where the dictionary is hard-coded during training, will most likely fail, as they cannot predict the updated support.  

Here and below, we compare the results of four solvers: ISTA, FISTA, Oracle-LISTA and Ada-LISTA, all versus the number of iterations/unfoldings, $K$. For each training example in ISTA, FISTA and Ada-LISTA, we create new instances of a permuted dictionary $\tilde{\rmD_i}$ and its corresponding true representation, $\rvx^*_i$. We then apply FISTA for $100$ iterations and obtain the ground truth representations $\rvx_i$ for the signal $\rvy_i=\tilde{\rmD} \rvx^*_i$. Then ISTA and FISTA are applied for only $K$ iterations to solve for the pairs $\{\rvy_i, \tilde{\rmD}_i\}$. Similarly, the supervised Ada-LISTA is given the ground truth $\{\rvy_i, \tilde{\rmD}_i, \rvx_i\}$ for training. In Oracle-LISTA \emph{we solve a simpler problem} in which the dictionary is fixed ($\rmD$) for all training examples $\{\rvy_i,\rvx_i\}$. The results in \Figref{fig:results_permutations} clearly show that Ada-LISTA is much more efficient compared to ISTA/FISTA, capable of mimicking the performance of the Oracle-LISTA, which considers a single constant $\rmD$.

\begin{figure}[t]
    \centering
    \includegraphics[width=0.85\linewidth]{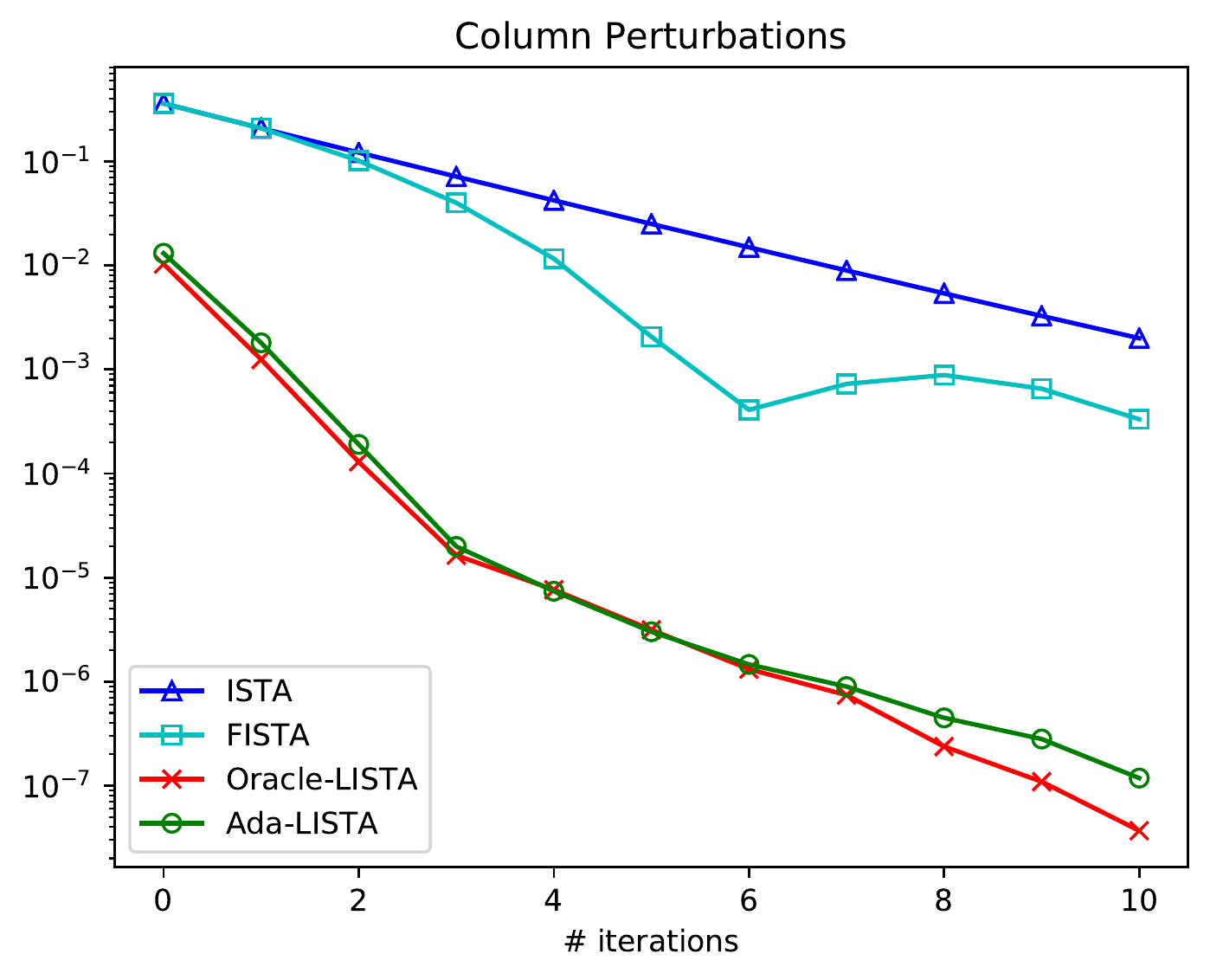}
    \caption{MSE performance under column permutations.}
    \label{fig:results_permutations}
\end{figure}

\paragraph{Noisy Dictionaries.} In this experiment we aim to show that Ada-LISTA can handle a more challenging case in which the dictionary varies by $\tilde{\rmD}_i = \rmD + \rmE_i$. Each training signal $\rvy_i$ is created by drawing a different noisy instance of the dictionary $\tilde{\rmD}_i$ and a sparse representation $\rvx^*_i$, and solving the FISTA to obtain $\rvx_i$. ISTA and FISTA receive the pairs $\{\rvy_i,\tilde{\rmD}_i\}$, and Ada-LISTA receives the triplet $\{\rvy_i,\tilde{\rmD}_i,\rvx_i\}$. By vanilla LISTA, we refer to a learned solver that obtains $\{\rvy_i,\rvx_i\}$, and trains a network while disregarding the changing models. Oracle-LISTA, as before, handles a simpler case in which the dictionary is fixed, being $\rmD$, and all signals are created from it.

Figure \ref{fig:results_noisy} presents the performance of the different solvers with a decreasing SNR (Signal to Noise Ratio) of the dictionary\footnote{Note that the noise is inflicted on the model (i.e., the dictionary) without an additive noise on the resulting signals}. The performance of  ISTA and FISTA is agnostic to the noisy model, since they do not require prior training. The Ada-LISTA again performs on-par with Oracle-LISTA, which has a prior knowledge of the dictionary. LISTA's performance, however, deteriorates with the decrease of the dictionary SNR. At $\text{SNR}=25$dB it still provides a computational gain over ISTA and FISTA, but loses its advantage for lower SNRs and higher number of iterations. 

\begin{figure*}[ht!]
    \centering
    \begin{subfigure}[t]{0.3\linewidth}
        \centering
        \includegraphics[width=1.05\linewidth]{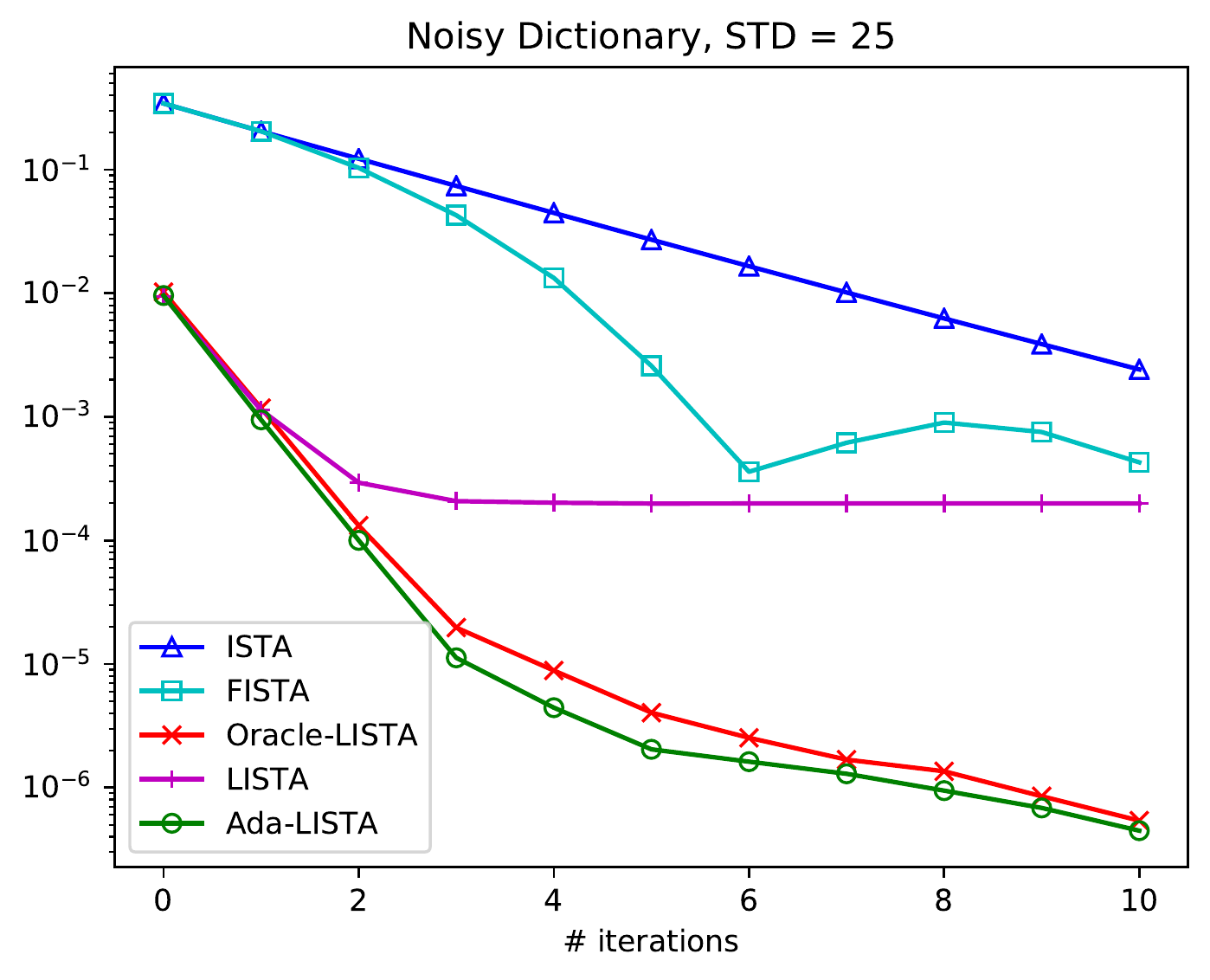}
        \caption{SNR of $25$[dB].}
        \label{fig:results_noisy_25}
    \end{subfigure}
    ~
    \begin{subfigure}[t]{0.3\linewidth}
        \centering
        \includegraphics[width=1.05\linewidth]{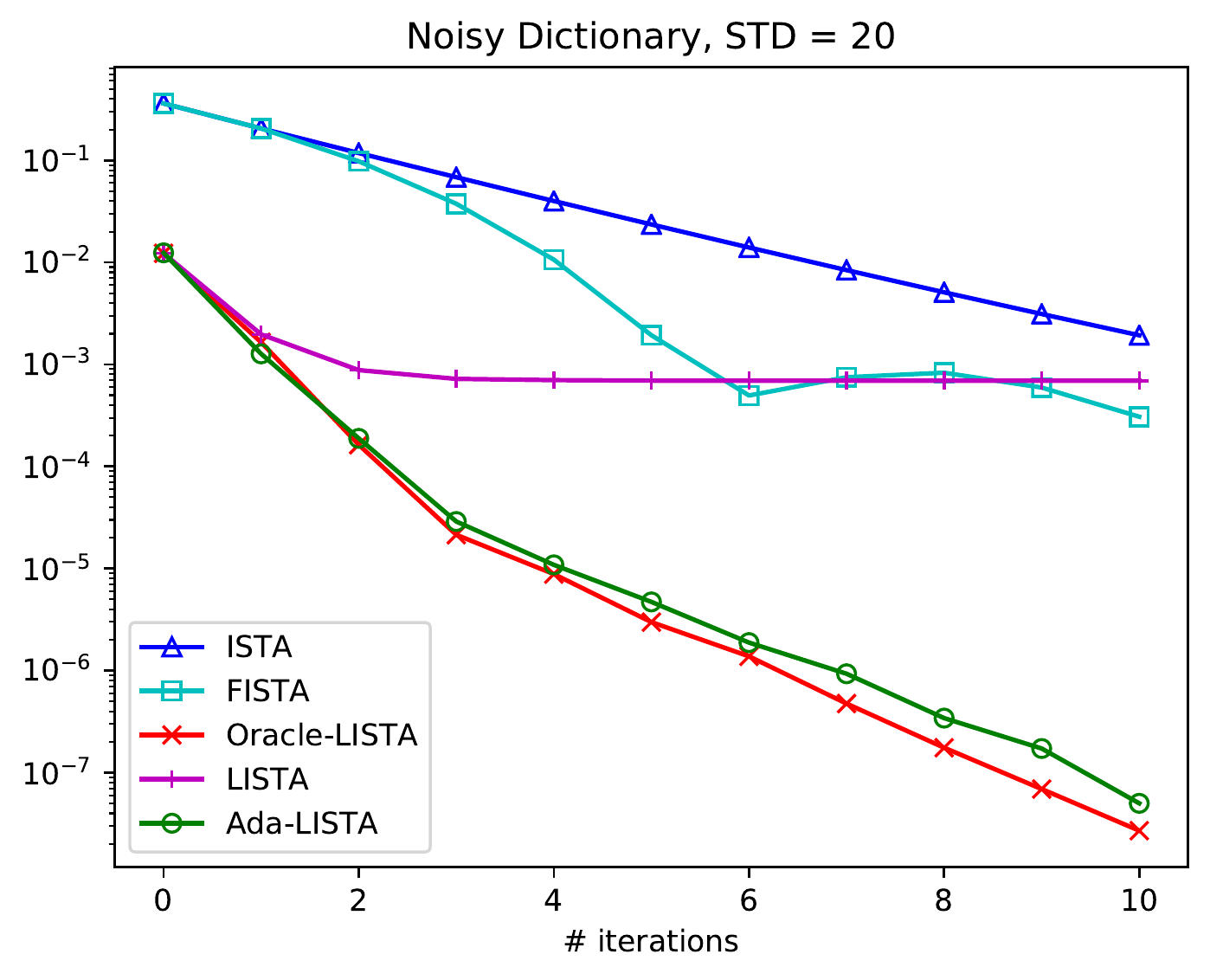}
        \caption{SNR of $20$[dB].}
        \label{fig:results_noisy_20}
    \end{subfigure}
    ~
    \begin{subfigure}[t]{0.3\linewidth}
        \centering
        \includegraphics[width=1.05\linewidth]{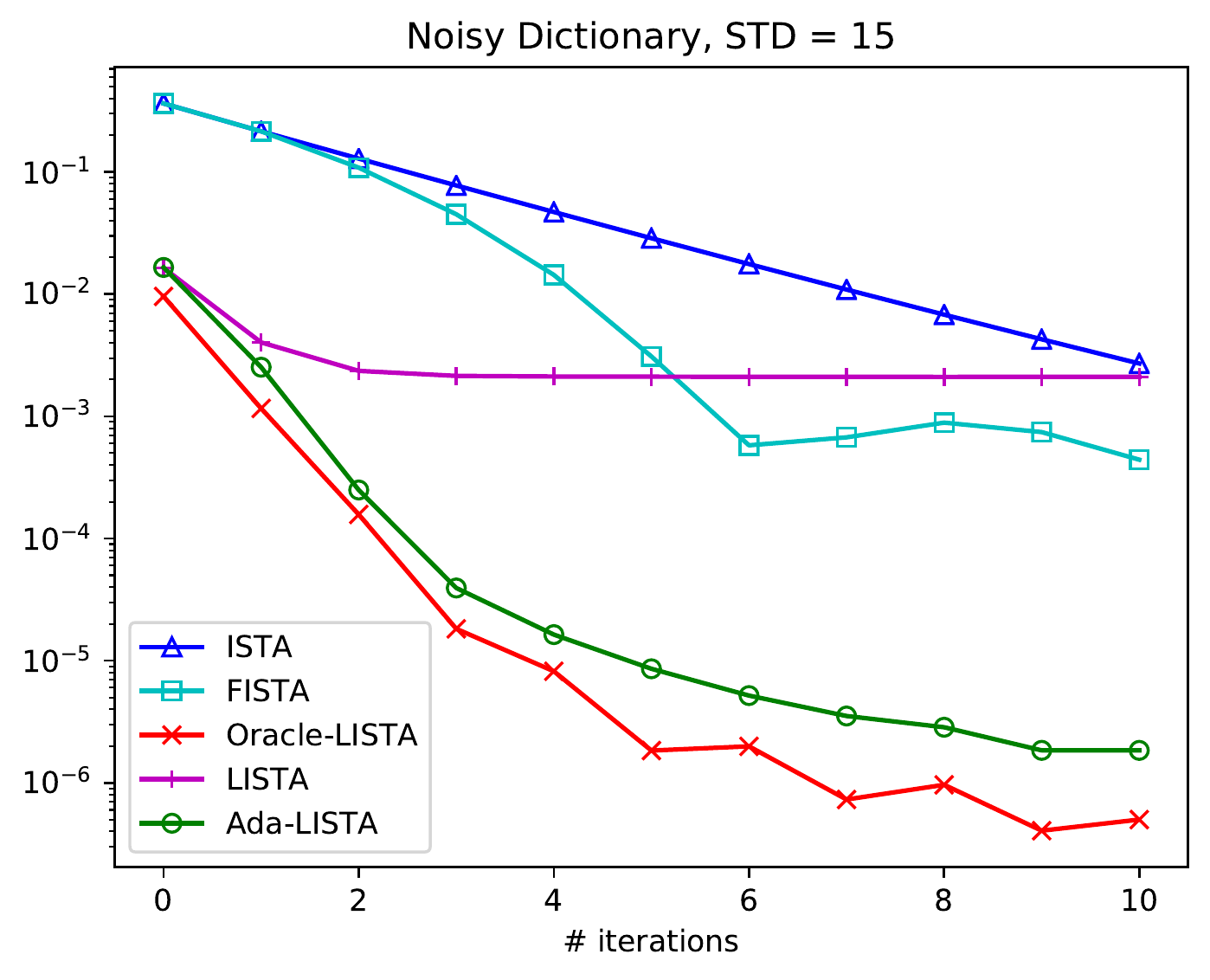}
        \caption{SNR of $15$[dB].}
        \label{fig:results_noisy_15}
    \end{subfigure}
    \caption{MSE performance for noisy dictionaries with decreasing SNR values.}
    \label{fig:results_noisy}
\end{figure*}

\begin{figure*}[ht!]
    \centering
    \begin{subfigure}[t]{0.3\linewidth}
        \centering
        \includegraphics[width=1.05\linewidth]{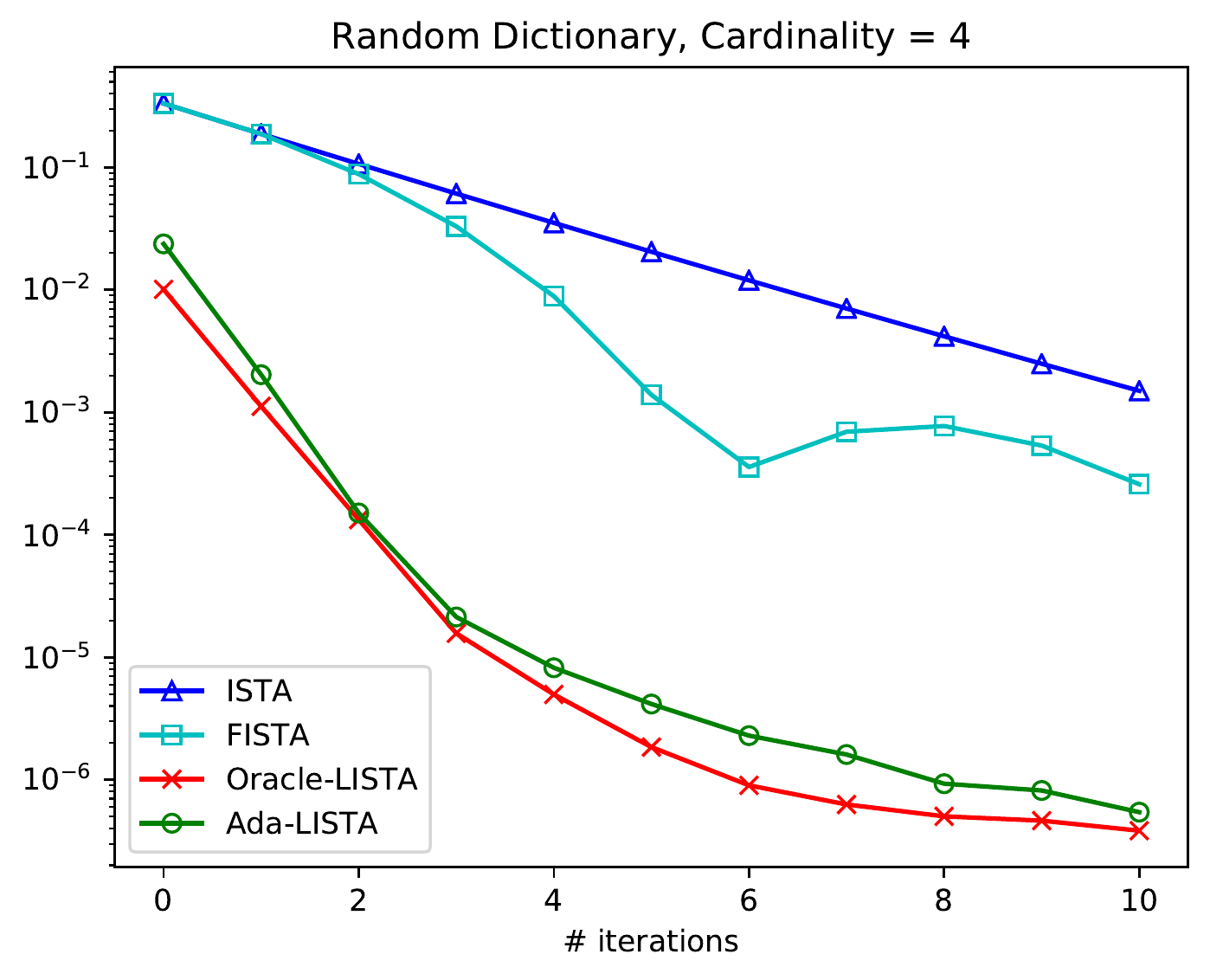}
        \caption{$s=4$.}
        \label{fig:results_random_s4}
    \end{subfigure}
    ~
    \begin{subfigure}[t]{0.3\linewidth}
        \centering
        \includegraphics[width=1.05\linewidth]{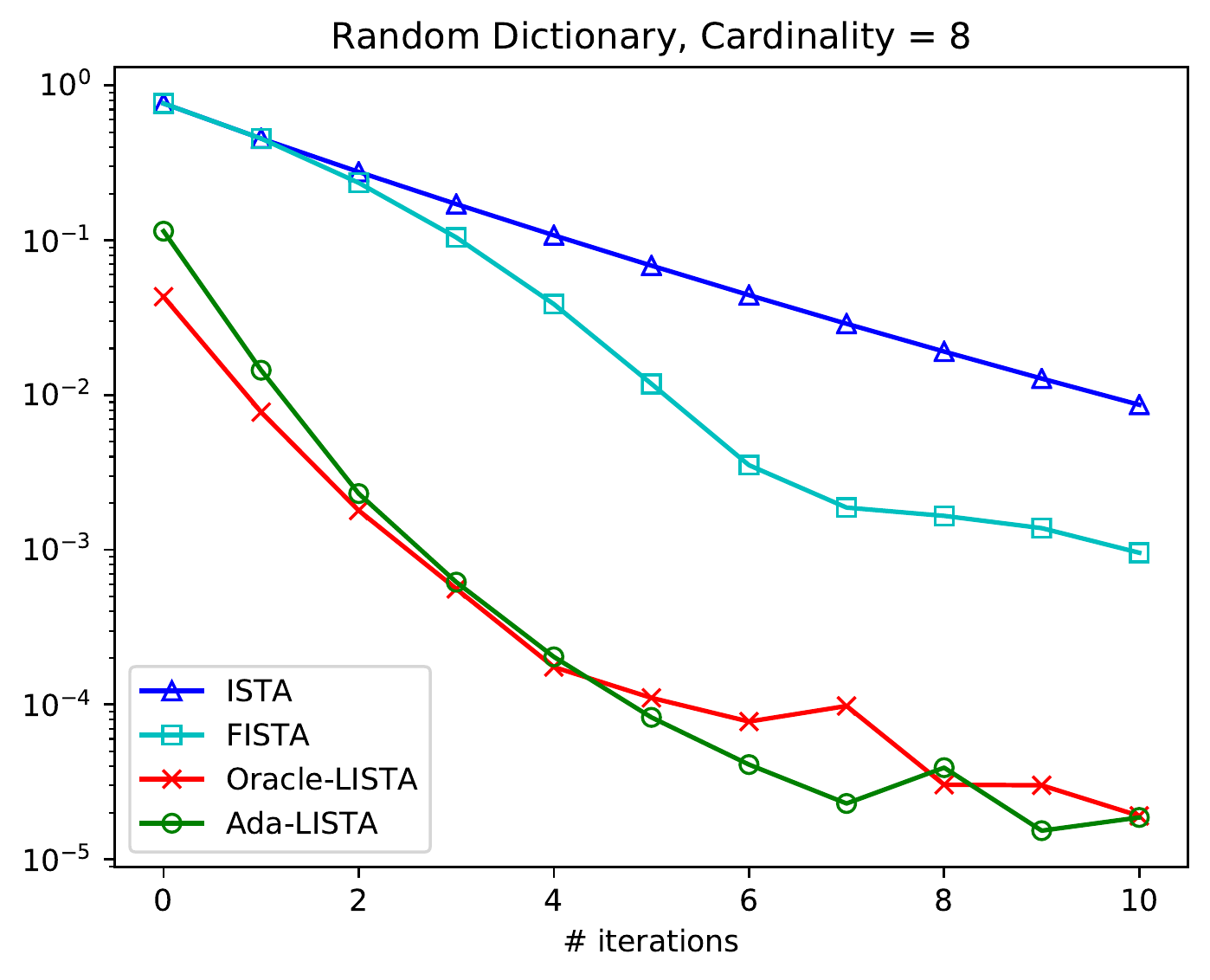}
        \caption{$s=8$.}
        \label{fig:results_random_s8}
    \end{subfigure}
    ~
    \begin{subfigure}[t]{0.3\linewidth}
        \centering
        \includegraphics[width=1.05\linewidth]{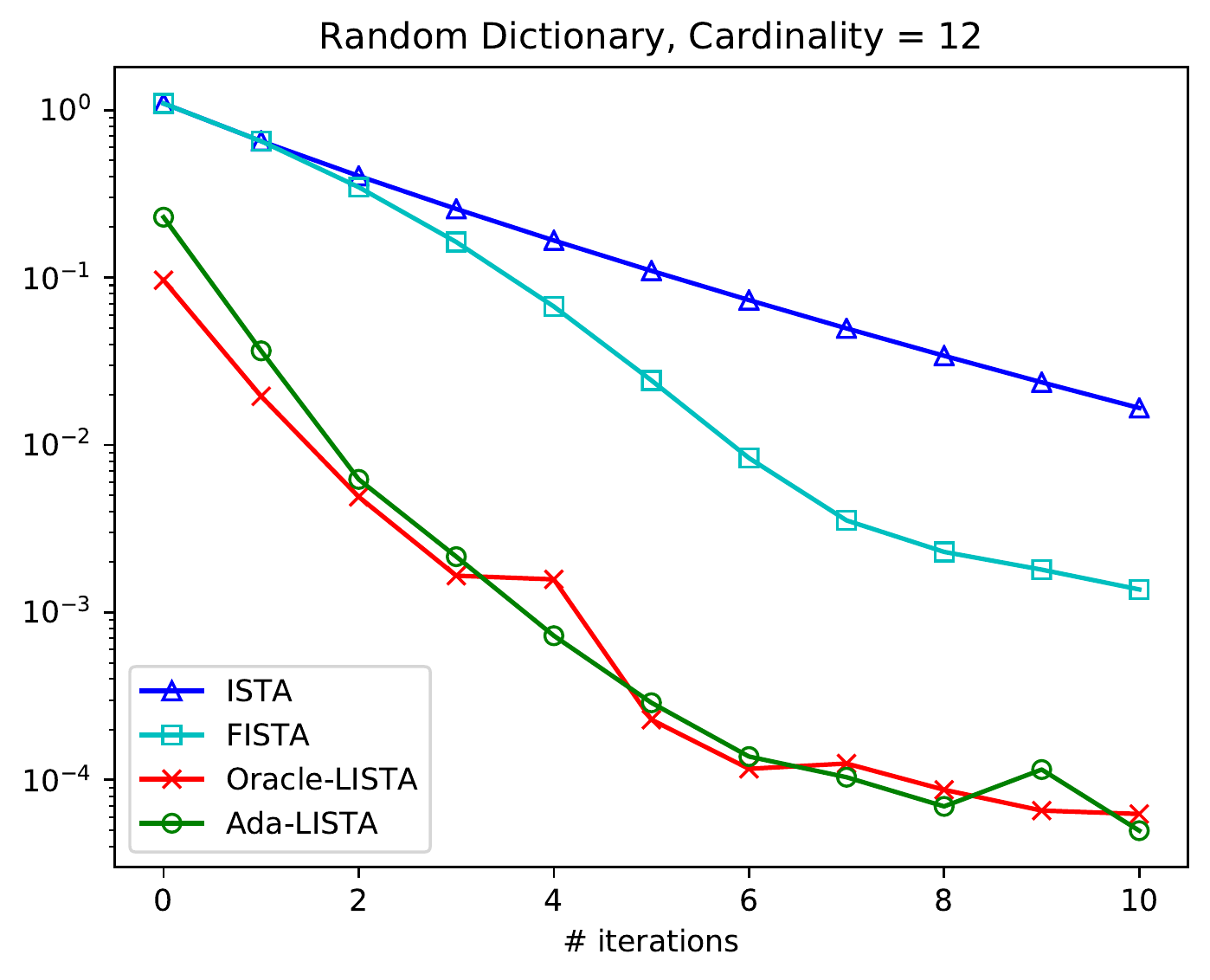}
        \caption{$s=12$.}
        \label{fig:results_random_s12}
    \end{subfigure}
    \caption{MSE performance for random dictionaries with increasing cardinality.}
    \label{fig:results_random}
\end{figure*}

\paragraph{Random Dictionaries.} In this setting, we diverge even further from a fixed model, and examine the capability of our method to handle completely random input dictionaries. This time, for each training example we create a different Gaussian normalized dictionary $\rmD_i$, and a corresponding representation vector with an increasing cardinality: $s=4,8,12$. The resulting signals, $\rvy_i=\rmD_i \rvx_i^*$, and their corresponding dictionaries are fed to FISTA to obtain the ground truth sparse representations for training, $\rvx_i$. We compare the performance of ISTA, FISTA, Ada-LISTA and Oracle-LISTA. Similarly to previous experiments, Ada-LISTA is fed during training with the triplet $\{\rvy_i,\rmD_i,\rvx_i\}_{i=1}^N$. Vanilla LISTA cannot handle such variation in the input distribution, and thus it is omitted. For reference, we show the results of Oracle-LISTA in which all of the training signals are created from the same dictionary.

As can be seen in Figure \ref{fig:results_random}, for a small cardinality of $s=4$, Oracle-LISTA is able to drastically lower the reconstruction error as compared to ISTA and FISTA. This result, however, has already been demonstrated in \cite{LISTA}. Ada-LISTA which deals with a much more complex scenario, still provides a similar  improvement over both ISTA and FISTA. As the cardinality increases to $s=8,12$, the performance of both learned solvers deteriorates, and the improvement over their non-learned counterparts diminishes. 

The last experiment provides a valuable insight on the success of LISTA-like learned solvers.
The common belief is that acceleration in convergence can be obtained when the signals are restricted to a union of low-dimensional subspaces, as opposed to the entire signal space. The above experiment suggests otherwise: Although the signals occupy the whole space, Ada-LISTA still achieves improved convergence. This implies that the underlying structure should be only of the \emph{signal given its generative model} $p(\rvy | \rmD)$, as opposed to the \emph{signal} model, $p(\rvy)$. In the above, even if the dictionaries are random, the signals must be \textit{sparse combinations of atoms}. As this assumption of structure weakens with the increased cardinality, the resulting acceleration becomes less prominent. We believe that this conditional information is the key for improved convergence.

\begin{figure*}[ht]
    \centering
    \inpaintingl{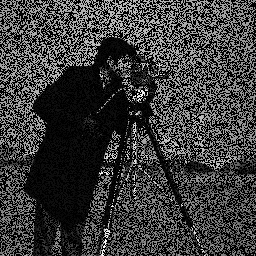} \inpaintingl{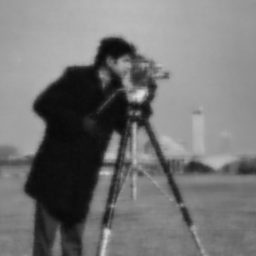}
    \inpaintingl{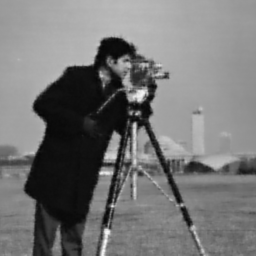} \inpaintingl{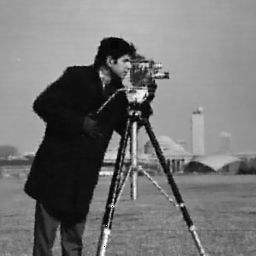}
    \caption{Image inpainting with $50\%$ missing pixels. From left to right: corrupted image, ISTA, FISTA, and Ada-LISTA.}
    \label{fig:inpainting_cameraman}
\end{figure*}


\subsection{Natural Image Inpainting} \label{subsec:inpainting}
In this section we apply our method to a natural image inpainting task. We assume the image is corrupted by a known mask with a ratio of $p$ missing pixels. Thus, the updated objective we wish to solve is 
\begin{equation} \label{eq:inpainting_objective}
    \minimize_{\rvx} \frac{1}{2} \|\rvy - \rmM \rmD \rvx\|_2^2 + \lambda \|\rvx\|_1,
\end{equation}
where $\rvy \in \R^n$ is a corrupt patch of the same size as the clean one, $\rmD \in \R^{n \times m}$ is a dictionary trained on clean image patches, and $\rmM \in \R^{n \times n}$ represents the mask, being an identity matrix with a percentage of $p$ diagonal elements equal to zero. Thus, the dictionary is constant, but each patch has a different (yet known) inpainting mask, and thus the effective dictionary $\rmD_{\text{eff}}=\rmM \rmD$ changes for each signal.

\begin{table*}[t!]
\begin{center}
\resizebox{\textwidth}{!}{
\begin{tabular}{l c c c c c c c c c c c}
\hline
\hline
 & Barbara & Boat & House & Lena & Peppers & C.man & Couple & Finger & Hill & Man & Montage \\
\hline
\textit{ISTA} & $23.49$ & $25.40$ & $26.87$ & $27.83$ & $23.56$ & $22.72$ & $25.34$ & $20.63$ & $27.26$ & $26.34$ & $22.48$  \\
\textit{FISTA} & $24.93$ & $28.18$ & $30.53$ & $31.02$ & $26.75$ & $25.25$ & $28.09$ & $25.45$ & $29.64$ & $29.03$ & $25.08$ \\
\textit{Ada-LFISTA} & $\textbf{26.09}$ & $\textbf{30.03}$ & $\textbf{32.36}$ & $\textbf{32.50}$ & $\textbf{28.81}$ & $\textbf{27.94}$ & $\textbf{30.02}$ & $\textbf{28.25}$ & $\textbf{30.86}$ & $\textbf{30.67}$ & $\textbf{27.22}$ \\
\hline
\hline
\end{tabular}}
\end{center}
\caption{Image inpainting with $50\%$ missing pixels and $K=20$ unfoldings.}\label{tbl:inpainting_results}
\end{table*}
\vspace{-1mm}

\begin{figure}[t]
    \centering
    \includegraphics[width=0.85\linewidth]{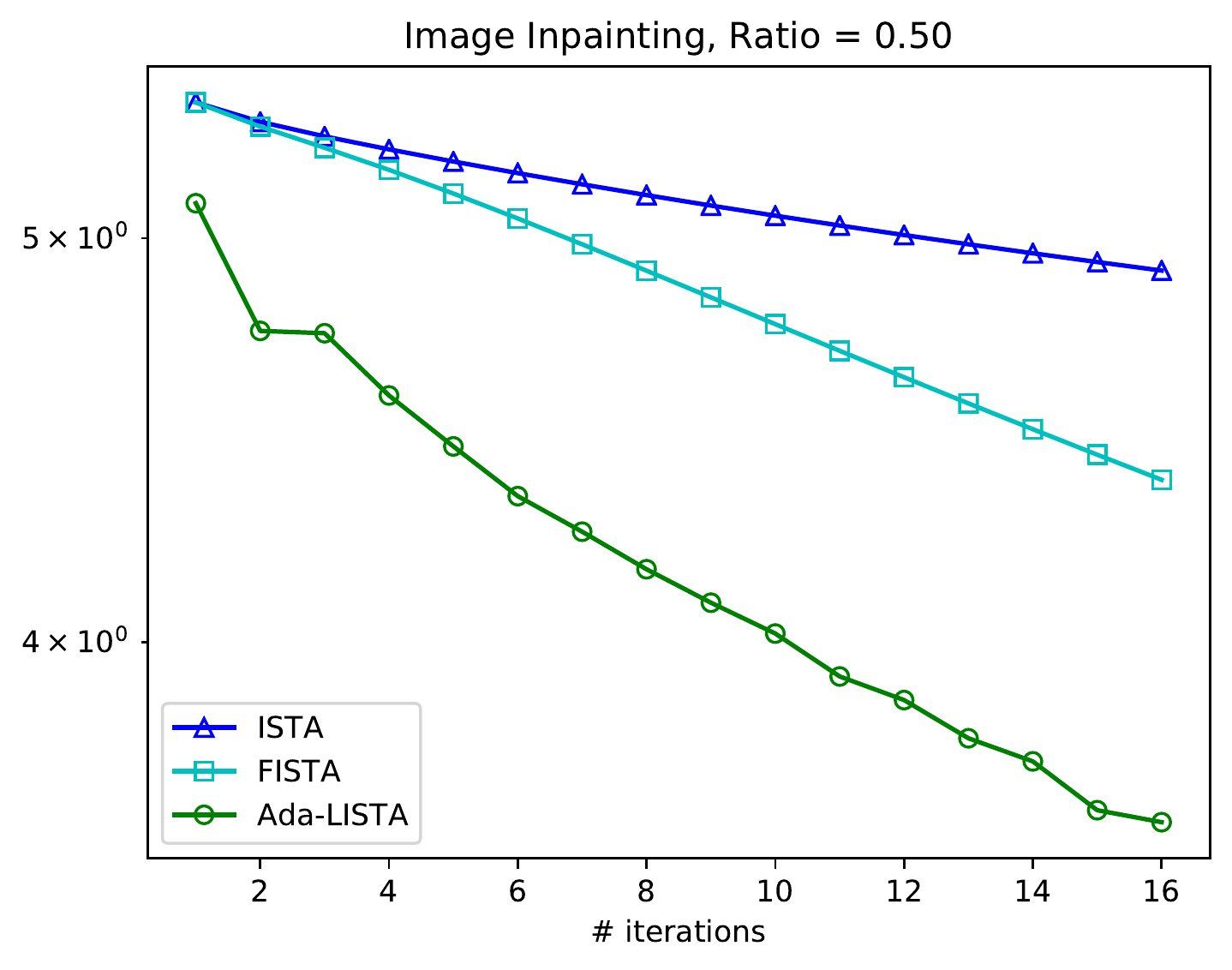}
    \caption{Patch-wise validation error versus unfoldings.}
    \label{fig:inpainting_convergence}
\end{figure}

\paragraph{Updated Model.} We slightly change the formulation of the model described in Section \ref{sec:our_method}, and reverse the roles of the input and learned matrices. Specifically, the updated shrinkage step (\Eqref{eq:ista_step}) for image inpainting is 
\begin{equation} \label{eq:inpainting_shrinkage}
    \rvx_{k+1}  = \gS_{\frac{\lambda}{L}}\left(\rvx_k + \frac{1}{L} \rmD^T \rmM^T (\rvy - \rmM \rmD \rvx_k) \right).
\end{equation}
We consider the mask $\rmM$ as part of the input, while the dictionary $\rmD$ is learned with the following parameterization:
\begin{equation}\label{eq:inpainting_model}
\begin{split} 
& \frac{1}{L} \rmD^T \rmM^T \rmM \rmD \rightarrow \gamma_{k+1} \rmW_1^T \rmM^T \rmM \rmW_1^T, \\
& \frac{1}{L} \rmD^T \rmM^T \rightarrow \gamma_{k+1} \rmW_2^T \rmM^T,
\end{split}
\end{equation}
where $\rmW_1,\rmW_2 \in \R^{n \times m}$ are the same size as the dictionary $\rmD$, and initialized by it.

\paragraph{Experiment Setting.} In order to collect natural image patches, we use the BSDS500 dataset \cite{BSDS500} and divide it to $400, 50$ and $50$ training, validation and test images correspondingly. To train the dictionary $\rmD$, we extract $100,000$ $8 \times 8$ patches at random locations from the train images, subtract their mean and divide by the average standard deviation. The dictionary of size $\rmD \in \R^{64 \times 256}$ is learned via \texttt{scikit-learn}'s function \texttt{MiniBatchDictionaryLearning} with $\lambda=0.1$. To train our network, we randomly pick a subset of $N=50,000$ training and $N_\text{val}=1,000$ validation patches. 
We train the network to perform an image inpainting task with ratio of $p=0.5$. Instead of using Ada-LISTA as before, we tweak the architecture described in Equation (\ref{eq:inpainting_model}) to unfold the FISTA algorithm, termed Ada-LFISTA, as described in algorithm \ref{alg:adalfista}. The input to our network is triplets $\{\rvy_i,\rmM_i,\rvx_i\}_{i=1}^N$ of the corrupt train patches $\rvy_i$, their corresponding mask $\rmM_i$, and the solutions $\rvx_i$ of the FISTA solver applied for $300$ iterations on the corrupt signals. The output is the reconstructed representations $\rvx_{K_i}$. 

We evaluate the performance of our method on images from the popular \texttt{Set11}, corrupted with the same inpainting ratio of $p=0.5$, and compare between ISTA, FISTA and Ada-LFISTA for a fixed number of $K=20$ iterations/unfoldings. We extract all overlapping patches in each image, subtract the mean and divide by the standard deviation, apply each solver, un-normalize the patches and return their mean, and then place them in their correct position in the image and average over overlaps. The quality of the results is measured in PSNR between the clean images and the reconstruction of their corrupt version. The patch-wise validation error versus the the number of unfoldings is given in Figure \ref{fig:inpainting_convergence}; numerical results are given in Table \ref{tbl:inpainting_results}, and select qualitative results are shown in Figure \ref{fig:inpainting_cameraman} and more in Appendix \ref{app:inpainting}. There is a clear advantage to Ada-LFISTA over the non-learned ISTA and FISTA solvers. In this setting of $50\%$ missing pixels, a hard-coded solver with a fixed $\rmD$, such as LISTA, cannot deal with the changing mask of each patch. 

\section{Conclusions}

We have introduced a new extension of LISTA, termed Ada-LISTA, which receives both the signals and their dictionaries as input, and learns a universal architecture that can cope with the varying models. This modification produces great flexibility in working with changing dictionaries, leveling the playing field with non-learned solvers such as ISTA and FISTA that are agnostic to the entire signal distribution, while enjoying the acceleration and convergence benefits of learned solvers. We have substantiated the validity of our method, both in a comprehensive theoretical study, and with extensive synthetic and real-world experiments. Future work includes further investigation of the discussed rationale, and an extension to additional applications.

\bibliography{bib}
\bibliographystyle{icml2020}


\appendix
\section{Proof of Theorem \ref{th:fixed_model}}
\label{ap:proof_1}

\begin{proof}
This proof follows the steps from \cite{zarka2019deep}, with slight modifications to fit our scheme. 
Following the notations in Theorem \ref{th:fixed_model}, $\rvx^*$ denotes the true sparse representation of the signal $\rvy$, and $\rmA=\rmW\rmD$. In addition, we define $\Supp(\cdot)$ as the support of a vector.

\paragraph{Induction hypothesis:} 
For any iteration $k\geq 0$ the following hold
\begin{enumerate}
    \item The estimated \emph{support} is contained in the true support,
    \begin{equation}\label{eq:induction_hyp_supp}
        \Supp(\rvx_k) \subseteq \Supp (\rvx^*).
    \end{equation}
    \item The \emph{recovery error} is bounded by
    \begin{equation}\label{eq:induction_hyp_recover}
        \|\rvx_k - \rvx^* \|_\infty \leq 2  \theta_k.
    \end{equation}
\end{enumerate}

\paragraph{Base case:}
We start by showing that the induction hypothesis holds for $k=0$. Since $\rvx_0 = \rvzero$ we get that the support is empty and the support hypothesis \Eqref{eq:induction_hyp_supp} holds. As for the recovery error, we get that
\begin{equation}
    \|\rvx_0 - \rvx^* \|_\infty = \|\rvx^*\|_\infty.
\end{equation}
Therefore, to verify \Eqref{eq:induction_hyp_recover} we need to show that
\begin{equation}
    \|\rvx^*\|_\infty \leq 2 \theta_0 = 2 \thmax.
\end{equation}
Since $\rvy = \rmD \rvx^* + \rve$, for any index $i$ we can write
\begin{equation}
    \rvy = \rvd_i \rvx^*[i] + \sum_{j \ne i} \rvd_j \rvx^*[j] + \rve,
\end{equation}
where $\rvd_i$ denotes the $i$th column in $\rmD$ and $\rvx[i]$ denotes the $i$th element in $\rvx$.
Multiplying each side by $\rva_i^T$ we get
\begin{equation}
    \rvx^*[i] \rva_i^T \rvd_i = \rva_i^T \rvy - \sum_{j \ne i} \rvx^*[j] \rva_i^T \rvd_j - \rva_i^T \rve.
\end{equation}
Since by assumption $\rva_i^T \rvd_i = 1$, the left term becomes $\rvx^*[i]$. In addition, since, by assumption, there are no more than $s$ nonzeros in $\rvx^*$ and $\abs{\rva_i^T \rvd_j}$ is bounded by $\tmu$, we get the following bound
\begin{equation}
    \abs{\rvx^*[i]} \leq \abs{\rva_i^T \rvy} +s \tmu \norm{\rvx^*}_{\infty} + \abs{\rva_i^T \rve}.
\end{equation}
By taking a maximum over $i$ we obtain
\begin{equation}\label{eq:27}
    (1 - s \tmu)\norminf{\rvx^*} \leq \norminf{\rmA^T \rvy} + \norminf{\rmA^T \rve}.
\end{equation}
Since we have assumed that $\norminf{\rmA^T \rvy} \leq \thmax$, and 
\begin{equation}
    \norminf{\rmA^T \rve} = \thmin (1-2\gamma \tmu s) < \thmax (1-2\gamma \tmu s),
\end{equation}
we get  
\begin{equation}
    (1 - s \tmu) \norminf{\rvx^*} \leq 2\thmax (1-\gamma \tmu s).
\end{equation}
Finally, since $s \tmu\leq \frac{1}{2}$, and $\gamma > 1$, we get 
\begin{equation}\label{eq:30}
    \norminf{\rvx_0 - \rvx^*} = \norminf{\rvx^*} \leq 2\thmax,
\end{equation}
as in \Eqref{eq:induction_hyp_recover}, and therefore the recovery error hypothesis holds for the base case.

\paragraph{Inductive step: }
Assuming the induction hypothesis holds for iteration $k$, we show that it also holds for the next iteration $k+1$. We define $\mathcal{I} \triangleq \Supp(\rvx^*) - \{i\}$ and denote by $\rmD_{\mathcal{I}}$ the subset $\mathcal{I}$ of columns in $\rmD$.

We start by proving the support hypothesis (\Eqref{eq:induction_hyp_supp}). By Definition \ref{def:adalista_one}, the following holds for any index $i$: 
\begin{equation}
    \rvx_{k+1}[i] = \gS_{\theta_{k+1}} ( \rvx_k[i] + \rva_i^T (\rvy - \rmD\rvx_k) ).
\end{equation}
Placing $\rvy = \rmD\rvx^* + \rve$, we get
\begin{equation} \label{eq:32}
    \rvx_{k+1}[i] = \gS_{\theta_{k+1}} ( \rvx_k[i] + \rva_i^T \rmD (\rvx^* - \rvx_k) + \rva_i^T \rve ).
\end{equation}
Since $\rva_i^T\rvd_i = 1$, the following holds:
\begin{equation*}
    \rvx_k[i] + \rva_i^T \rmD (\rvx^* - \rvx_k) = 
        \rvx^*[i] + \rva_i^T \rmD_{\mathcal{I}} (\rvx^*[\mathcal{I}] - \rvx_k[\mathcal{I}]).
\end{equation*}
Therefore, \Eqref{eq:32} becomes
\begin{equation}\label{eq:34}
    \rvx_{k+1}[i] = \gS_{\theta_{k+1}} ( \underbrace{\rvx^*[i] + \rva_i^T \rmD_{\mathcal{I}} (\rvx^*[\mathcal{I}] - \rvx_k[\mathcal{I}]) + \rva_i^T \rve}_{\triangleq r} ).
\end{equation}

We aim to show that for any $i \notin \Supp(\rvx^*)$, $\rvx_{k+1}[i]=0$, as the support hypothesis suggests. Since $\rvx^*[i]=0$, we can bound the input argument of the soft threshold by
\begin{equation}
    \abs{r} \leq \abs{\rva_i^T \rmD_{\mathcal{I}} (\rvx^*[\mathcal{I}] - \rvx_k[\mathcal{I}])} + \norminf{\rmA^T \rve}.
\end{equation}
Using the induction assumption on the support, $\Supp(\rvx_k) \in \Supp(\rvx^*)$, we can upper bound the first term in the right-hand-side, 
\begin{equation}
    \abs{\rva_i^T \rmD_{\mathcal{I}} (\rvx^*[\mathcal{I}] - \rvx_k[\mathcal{I}])} \leq s \tmu \norminf{\rvx^* - \rvx_k}.
\end{equation}
Using the induction assumption on the recovery error (\Eqref{eq:induction_hyp_recover}), we have $\norminf{\rvx^* - \rvx_k} \leq 2\theta_k$. Therefore, we get  
\begin{equation}
    \abs{r} \leq 2s \tmu \theta_k + \norminf{\rmA^T \rve}.
\end{equation}
However, by our assumptions, 
\begin{equation}\label{eq:38}
    \norminf{\rmA^T \rve} = \thmin (1-2\gamma \tmu s) < \theta_{k+1} (1-2\gamma \tmu s).
\end{equation}
Therefore,
\begin{equation}
    \abs{r} \leq 2s \tmu \theta_k + \theta_{k+1} (1-2\gamma \tmu s),
\end{equation}
and by placing $\theta_k = \gamma \theta_{k+1}$ we get  
\begin{equation}
    \abs{r} \leq \theta_{k+1}.
\end{equation}
Since $r$ is the input to the soft threshold operator $\gS_{\theta_{k+1}}$, and it is no bigger than the threshold, we get that $\rvx_{k+1}[i]=0$, and the support hypothesis holds.

We proceed by proving that the recovery error hypothesis also holds (\Eqref{eq:induction_hyp_recover}). We use the fact that for any scalar triplet, $(\rvx_1, \rvx_2, \theta)$, the soft threshold satisfies
\begin{equation}\label{eq:41}
    \abs{\gS_{\theta}(\rvx_1 + \rvx_2) - \rvx_1} \leq \theta + \abs{\rvx_2}.
\end{equation}
Therefore, following \Eqref{eq:34} we get 
\begin{multline*}
    \abs{\rvx_{k+1}[i] - \rvx^*[i]} \leq \\
    \theta_{k+1} + \abs{\rva_i^T \rmD_{\mathcal{I}} (\rvx^*[\mathcal{I}] - \rvx_k[\mathcal{I}])} + \norminf{\rmA^T \rve}.
\end{multline*}
As before, since $\Supp(\rvx_k) \in \Supp(\rvx^*)$, we have 
\begin{equation}
    \abs{\rva_i^T \rmD_{\mathcal{I}} (\rvx^*[\mathcal{I}] - \rvx_k[\mathcal{I}])} \leq 2 s \tmu \theta_k.
\end{equation}
Therefore, by using \Eqref{eq:38} we get 
\begin{equation}
    \abs{\rvx_{k+1}[i] - \rvx^*[i]} \leq \theta_{k+1} + 2 s \tmu \theta_k + \theta_{k+1} (1-2\gamma \tmu s),
\end{equation}
and by placing $\theta_k = \gamma \theta_{k+1}$ we obtain 
\begin{equation}
    \abs{\rvx_{k+1}[i] - \rvx^*[i]} \leq 2\theta_{k+1}.
\end{equation}
By taking a maximum over $i$, we establish the recovery error hypothesis (\Eqref{eq:induction_hyp_recover}), concluding the proof.
\end{proof}

\section{Proof of Theorem \ref{th:dictionaries_set}}
\label{ap:proof_dictionaries_set}

We define an effective matrix $\rmG = \rmD^T \rmW^T \rmD$. In this part, we aim to prove that linear convergence is guaranteed for any dictionary $\rmD$, satisfying two conditions: (i) the diagonal elements of $\rmG$ are close to $1$, and (ii) the off-diagonal elements of $\rmG$ are bounded.

\begin{proof}

This proof is based on Appendix \ref{ap:proof_1}, with the following two modifications: The mutual coherence $\tmu$ is replaced with $\bmu$, and the diagonal element $\rva_i^T \rvd_i$ is not assumed to be equal to $1$, but rather bounded from below by $1-\epsilon_d$.

The base case of the induction (\Eqref{eq:27}) now becomes:
\begin{equation}
    \norminf{\rvx^*} (1 - \epsilon_d - \bmu s) \leq \norminf{\rmA^T\rvy} + \norminf{\rmA^T \rve}.
\end{equation}
Since we assume $\norminf{\rmA^T\rvy} \leq \thmax$, and
\begin{equation}
    \norminf{\rmA^T \rve} < \thmax (1 - 2\gamma \epsilon_d - 2 \gamma \bmu s),
\end{equation}
we get 
\begin{equation}
    \norminf{\rvx^*} (1 - \epsilon_d - \bmu s) \leq 2 \thmax (1 - \gamma \epsilon_d - \gamma \bmu s).
\end{equation}
As $\gamma > 1$, $\norminf{\rvx^*}< 2\thmax$, therefore the induction hypothesis holds for the base case.

Moving to the inductive step, the proof of the support hypothesis remains almost the same, apart from replacing $\tmu$ with $\bmu$.
This is due to the fact that if $i \notin \Supp(\rvx^*)$, then $\rvx^*[i]=\rvx_k[i]=0$, and therefore the diagonal elements $\rva^T_i\rvd_i$ multiply zero elements.

As to the recovery error hypothesis, we need to upper bound $\norminf{\rvx^* - \rvx_{k+1}}$ for $i \in \Supp(\rvx^*)$. 
Since $\rva_i^T \rvd_i\ne 1$ we need to modify \Eqref{eq:32}: 
\begin{multline}
    \rvx_{k+1} [i] = \gS_{\theta_{k+1}} \big( \rvx^* [i] + 
    \rva_i^T \rmD_{\mathcal{I}} 
    (\rvx^* - \rvx_k)_{\mathcal{I}}  + \rva_i^T \rve \\
    + (1-\rva_i^T \rvd_i) (\rvx_k[i] - \rvx^* [i]) \big).
\end{multline}
Using \Eqref{eq:41} we get that $\abs{\rvx_{k+1} [i] - \rvx^* [i]}$ is upper bounded by
\begin{multline}
    \theta_{k+1}  + \bmu s  \| \rvx^* - \rvx_k\|_\infty 
    + \|\trmA^T \rve \|_\infty \\
    + \abs{(1-\trva_i^T \rvd_i)} \abs{ \rvx_k[i] - \rvx^* [i]},
\end{multline}
which in turn is upper bounded by
\begin{equation}
    \theta_{k+1}  +  \bmu s  2 \theta_k  + \theta_{k+1} (1 - 2 \gamma \epsilon_d - 2 \gamma \bmu s) + 2\epsilon_d \theta_k.
\end{equation}
Placing $\theta_k = \gamma \theta_{k+1}$ results in
\begin{equation}
    \abs{\rvx_{k+1} [i] - \rvx^* [i]} \leq 2 \theta_{k+1}.
\end{equation}
Taking a maximum over $i$ establishes the recovery error assumption, proving the induction hypothesis. 
\end{proof}

\section{Proof for Random Permutations}
\label{ap:proof_permutation}

We show that if the weight matrix $\rmW$ leads to linear convergence for signals generated by $\rmD$, then linear convergence is also guaranteed for signals originating from $\trmD = \rmD \rmP$, where $\rmP$ is a permutation matrix. The proof is straightforward, as the permutation matrix does not flip diagonal and off-diagonal elements in the effective matrix $\rmP^T \rmG \rmP$. Thus, the mutual coherence does not change and the conditions of Theorem \ref{th:dictionaries_set} hold, establishing linear convergence.

\section{Proof for Noisy Dictionaries}
\label{ap:proof_noisy_dict}

We now consider signals from noisy models, $\rvy = \trmD \rvx^* + \rve$, where $\trmD = \rmD + \rmE$, and the model deviations are of Gaussian distribution, $\emE_{i j}\sim \mathcal{N}(0,\sigma^2)$. Given pairs of $(\rvy, \trmD$), we show that Ada-LISTA recovers the original representations $\rvx^*$, with respect to their model $\trmD$ in linear rate. 

\begin{theorem}[Ada-LISTA Convergence -- Noisy Model]
\label{th:noisy_model}
    Consider a noisy input $\rvy = \trmD \rvx^* + \rve$, where $\trmD = \rmD + \rmE$, $\emE_{i j}\sim \mathcal{N}(0,\sigma^2/n)$. If for some constants $\taud,\tauod>0$, $\rvx^*$ is sufficiently sparse,
    \begin{equation}
        s = \norm{\rvx^*}_0 < \frac{1}{2 \bmu}, \quad \bmu \triangleq \tmu + \tauod,
    \end{equation}
    and the thresholds satisfy
    \begin{equation}
        \theta_k =  \thmax\, \gamma^{-k} > \thmin = \frac{\|\trmA^T \rve \|_\infty } {1 - 2\gamma\epsilon_d - 2 \gamma \bmu s},
    \end{equation}
    with $1 < \gamma < (2 \bmu s)^{-1}$, $\epsilon_d \triangleq \wdd + \taud <\frac{1}{2}$, $\wdd \triangleq \frac{\sigma^2}{n} \sum_{k=1}^n \emW_{k k}$, $\trmA \triangleq \rmW \trmD$, and $\thmax \geq \|\trmA^T \rvy\|_\infty$, then, with probability of at least $(1-p_1 p_2)$, the support of the $k$th iteration of Ada-LISTA is included in the support of $\rvx^*$ and its values satisfy
    \begin{equation}
        \|\rvx_k - \rvx^* \|_\infty \leq 2 \,\thmax \,\gamma^{-k}.
    \end{equation}
\end{theorem}

\begin{proof}

The proof for this theorem consists of two stages. First, we study the effect of model perturbations on the effective matrix $\trmG = \trmD^T\rmW^T\trmD$, deriving probabilistic bounds for the changes in the diagonal and off-diagonal elements. Then, we place these bounds in Theorem \ref{th:dictionaries_set} to guarantee linear rate.

We start by bounding the changes in the effective matrix $\trmG = \trmD^T\rmW^T\trmD$. These deviations modify the off-diagonal elements, which are no longer bounded by $\tmu$, and the diagonal elements that are not equal to $1$ anymore. 
Define $\brmG$ as:
\begin{equation}
    \brmG = \trmG - \rmG = \rmD^T \rmW^T \rmE + \rmE^T \rmW^T \rmD + \rmE^T \rmW^T \rmE.
\end{equation}
This implies $\bemG_{i j}$ is equal to:
\begin{multline}
    \bemG_{i j} = \underbrace{\sum_{k=1}^n \sum_{l=1}^n \emD_{k i} \emW_{l k} \emE_{l j}}_{\triangleq \emT^a_{i j}} + \underbrace{\sum_{k=1}^n \sum_{l=1}^n \emE_{k i} \emW_{l k} \emD_{l j}}_{\triangleq \emT^b_{i j}} \\ 
    + \underbrace{\sum_{k=1}^n \sum_{l=1}^n \emE_{k i} \emW_{l k} \emE_{l j}}_{\triangleq \emT^c_{i j}}.
\end{multline}

Since $\E[\emE_{i j}^2] = \frac{\sigma^2}{n}$ and the elements in $\rmE$ are independent, the expected value of $\bemG_{i j}$ is 
\begin{equation}
    \E[\bemG_{i j}] = 
    \begin{cases}
    \frac{\sigma^2}{n} \sum_{k=1}^n \emW_{k k}, & \mbox{if } i = j\\
    0, & \mbox{if } i \ne j.
    \end{cases}
\end{equation}

To bound the changes in $\bemG_{i j}$ we aim to use Cantelli's inequality, but first, we need to find the variance of $\bemG_{i j}$:
\begin{multline*}
    \Var [\bemG_{i j} ] = \E[\emT^a_{i j}]^2 + \E[\emT^b_{i j}]^2 + \E[\emT^c_{i j} - \E[\bemG_{i j}]]^2 + 2\E[\emT^a_{i j} \emT^b_{i j}] \\ + 2\E[\emT^a_{i j} (\emT^c_{i j} - \E[\bemG_{i j}])] + 2\E[\emT^b_{i j} (\emT^c_{i j} - \E[\bemG_{i j}])].
\end{multline*}
In what follows we calculate each term in the right-hand-side, starting with $\E[\emT^a_{i j}]^2$:
\begin{equation}
\begin{split}
    \E[\emT^a_{i j}]^2 & = \E \Big[ \sum_{k=1}^n \sum_{l=1}^n \emD_{k i} \emW_{l k} \emE_{l j} \sum_{k'=1}^n \sum_{l'=1}^n \emD_{k' i} \emW_{l' k'} \emE_{l' j} \Big] \\ 
    & = \frac{\sigma^2}{n} \underbrace{\E \Big[ \sum_{k=1}^n \sum_{l=1}^n \sum_{k'=1}^n \emD_{k i} \emW_{l k} \emD_{k' i} \emW_{l k'} \Big]}_{\triangleq \emC^a_{i j}}.
    \end{split}
\end{equation}
Moving on to $\E[\emT^b_{i j}]^2$, we get
\begin{equation}
\begin{split}
    \E[\emT^b_{i j}]^2 & = \E \Big[ \sum_{k=1}^n \sum_{l=1}^n \emE_{k i} \emW_{l k} \emD_{l j} \sum_{k'=1}^n \sum_{l'=1}^n \emE_{k' i} \emW_{l' k'} \emD_{l' j} \Big] \\ 
    & = \frac{\sigma^2}{n} \underbrace{\E \Big[ \sum_{k=1}^n \sum_{l=1}^n \sum_{l'=1}^n \emW_{l k} \emD_{l j} \emW_{l' k} \emD_{l' j} \Big]}_{\triangleq \emC^b_{i j}}.
    \end{split}
\end{equation}
As for $\E[\emT^c_{i j} - \E[\bemG_{i j}]]^2$, if $i \ne j$ then
\begin{equation}
\begin{split}
    \E \Big[ \sum_{k=1}^n \sum_{l=1}^n \emE_{k i} \emW_{l k} \emE_{l j} \sum_{k'=1}^n \sum_{l'=1}^n \emE_{k' i} \emW_{l' k'} \emE_{l' j}
    \Big] \\
    = \frac{\sigma^4}{n^2} \underbrace{\E \Big[ \sum_{k=1}^n \sum_{l=1}^n \emW_{l k}^2 \Big]}_{\triangleq \emC^c_{i j}}.
\end{split}
\end{equation}
Whereas, if $i = j$, then $\E[\emT^c_{i j} - \E[\bemG_{i j}]]^2$ becomes
\begin{multline}
    \E \Big[ \sum_{k=1}^n \sum_{l=1 \ne k}^n \emE_{k i}^2 \emW_{l k}^2 \emE_{l j}^2 \Big] + \E \Big[ \sum_{k=1}^n \sum_{l=1 \ne k}^n \emE_{k i}^2 \emW_{l k} \emW_{k l} \emE_{l j}^2 \Big] \\ + \E \Big[ \sum_{k=1}^n \sum_{k'=1 \ne k}^n \emE_{k i} \emW_{k k} \emE_{k j} \emE_{k' i} \emW_{k' k'} \emE_{k' j} \Big]\\
    + \E \Big[ \sum_{k=1}^n \emE_{k i}^2 \emW_{k k}^2 \emE_{k j}^2
    \Big] - \frac{\sigma^4}{n^2} \Big(\sum_{k=1}^n \emW_{k k} \Big)^2
\end{multline}
Using the fourth moment of Gaussian distribution, we obtain $\E[\emT^c_{i j} - \E[\bemG_{i j}]]^2$ is equal to
\begin{equation}
    \frac{\sigma^4}{n^2} \Bigg( \underbrace{\sum_{k=1}^n \sum_{l=1 \ne k}^n \emW_{l k}^2 + \sum_{k=1}^n \sum_{l=1 \ne k}^n \emW_{l k} \emW_{k l}
    + 2\sum_{k=1}^n \emW_{k k}^2}_{\triangleq \emC^d_{i j}} \Bigg).
\end{equation}
Continuing with $2\E[\emT^a_{i j} \emT^b_{i j}]$, we get
\begin{equation}
    2\E[\emT^a_{i j} \emT^b_{i j}] = 2\E \Big[ \sum_{k=1}^n \sum_{l=1}^n \sum_{l'=1}^n  \emD_{k i} \emW_{l k} \emE_{l j} \emE_{l i} \emW_{l' l} \emD_{l' j} \Big].
\end{equation}
Therefore, if $i = j$ then $\E[\emT^a_{i j} \emT^b_{i j}] = 0$, and if $i \ne j$ then
\begin{equation}
    2\E[\emT^a_{i j} \emT^b_{i j}] = 
    \frac{\sigma^2}{n} \underbrace{2\sum_{k=1}^n \sum_{l=1}^n \sum_{l'=1}^n  \emD_{k i} \emW_{l k} \emW_{l' l} \emD_{l' j}}_{\triangleq \emC^e_{i j}}.
\end{equation}
As for $2\E[\emT^a_{i j} (\emT^c_{i j} - \E[\bemG_{i j}])]$, and $2\E[\emT^b_{i j} (\emT^c_{i j} - \E[\bemG_{i j}])]$, both are zero since the third moment of Gaussian variable is zero.

To conclude, we define the maximal variance of the off-diagonal elements as,
\begin{equation}
    \Vod \triangleq \max_{i \ne j} \frac{\sigma^2}{n} \left(\emC^a_{i j} + \emC^b_{i j}\right) + \frac{\sigma^4}{n^2} \emC^c_{i j},
\end{equation}
and the maximal variance of the diagonal elements as,
\begin{equation}
    \Vd \triangleq \max_{i = j}  \frac{\sigma^2}{n} \left(\emC^a_{i j} + \emC^b_{i j} + \emC^e_{i j} \right) + \frac{\sigma^4}{n^2} \emC^d_{i j}.
\end{equation}

Identifying the variance of $\bemG_{i j}$ enables to bound the changes in the effective matrix using Cantelli's inequality. Starting with the off-diagonal elements, we obtain
\begin{equation}
    p(\abs{\bemG_{i j}} \geq \tauod ) \leq
    \frac{2\Vod^2}{\Vod^2 + \tauod^2}.
\end{equation}
Taking the maximum over all off-diagonal elements, we get
\begin{equation}
    p( \max_{i,j\ne i} \abs{\bemG_{i j}} \geq \tauod) \leq p_1,
\end{equation}
with
\begin{equation}
    p_1 \triangleq 1 -  \left(\frac{\tauod^2 - \Vod^2}{\Vod^2 + \tauod^2}\right)^{n(n-1)}.
\end{equation}
Moving on to the diagonal elements, we have
\begin{equation}
    p\left(\abs{\bemG_{i i} - \frac{\sigma^2}{n} \sum_{k=1}^n \emW_{k k}} \geq \taud \right) \leq
    \frac{2\Vd^2}{\Vd^2 + \taud^2}.
\end{equation}
Taking the maximum over all diagonal elements, we get
\begin{equation}
    p\left( \max_{i} \abs{\bemG_{i i} - \frac{\sigma^2}{n} \sum_{k=1}^n \emW_{k k}} \geq \taud \right) \leq p_2,
\end{equation}
with 
\begin{equation}
    p_2 \triangleq 1 -  \left(\frac{\taud^2 - \Vd^2}{\Vd^2 + \taud^2}\right)^{n}.
\end{equation}

Therefore, with probability of at least $1-p_1 p_2$, we obtain that the matrix $\trmG = \rmW^T \trmD^T \trmD$ satisfies the following:
\begin{itemize}
    \item The off-diagonal elements are bounded:
        \begin{equation}\label{eq:off_diag_bound}
            \max_{i,j\ne i} \abs{\etrmG_{i j}} \leq \tmu + \tauod.
        \end{equation}
    \item The diagonal elements are close to $1$:
        \begin{equation}\label{eq:diag_bound}
            \max_{i} \abs{\etrmG_{i i} - 1} \geq \wdd + \taud, \quad \wdd \triangleq \frac{\sigma^2}{n} \sum_{k=1}^n \emW_{k k}.
        \end{equation}
\end{itemize}

Finally, we apply Theorem \ref{th:dictionaries_set} with the constants
\begin{equation}
    \bmu = \tmu + \tauod, \quad
    \epsilon_d = \wdd + \taud,
\end{equation}
and establish linear convergence, with probability of at least $(1 - p_1 p_2$).

\end{proof}

\section{Synthetic Experiments on Noisy Signals}
\label{app:noisy_sig}

In this part we examine Ada-LISTA's performance for noisy signals by repeating the synthetic experiments in subsection \ref{subsec:synthetic}, with three levels of input SNR: $10$, $20$, and $30$[dB]. Figures \ref{fig:results_column_noisy}, \ref{fig:results_noisy_dict_noisy_sig}, and \ref{fig:results_random_dict_noisy_sig} respectively, present the results for column permutations, noisy dictionaries and random dictionaries. 
The same observations as in the noiseless case hold for noisy signals just as well. Learned solvers can achieve an acceleration even in the presence of noise in the input, and Ada-LISTA manages to mimic the oracle-LISTA, while  coping with a much harder scenario of varying dictionaries.

\begin{figure*}[ht!]
    \centering
    \begin{subfigure}[t]{0.3\linewidth}
        \centering
        \includegraphics[width=1.05\linewidth]{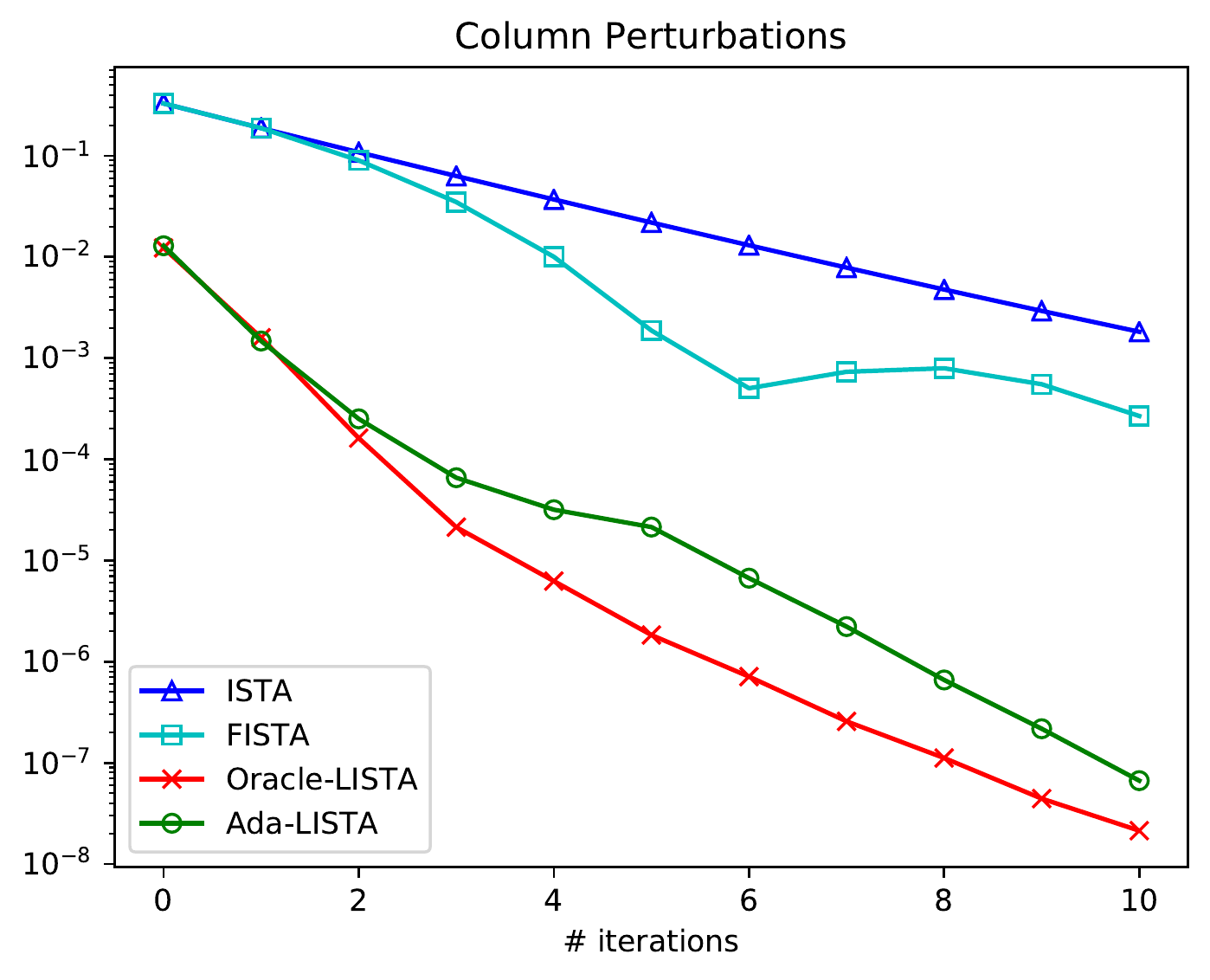}
        \caption{Signal SNR: $10$[dB].}
        \label{fig:results_column_noisy_snr_10}
    \end{subfigure}
    ~
    \begin{subfigure}[t]{0.3\linewidth}
        \centering
        \includegraphics[width=1.05\linewidth]{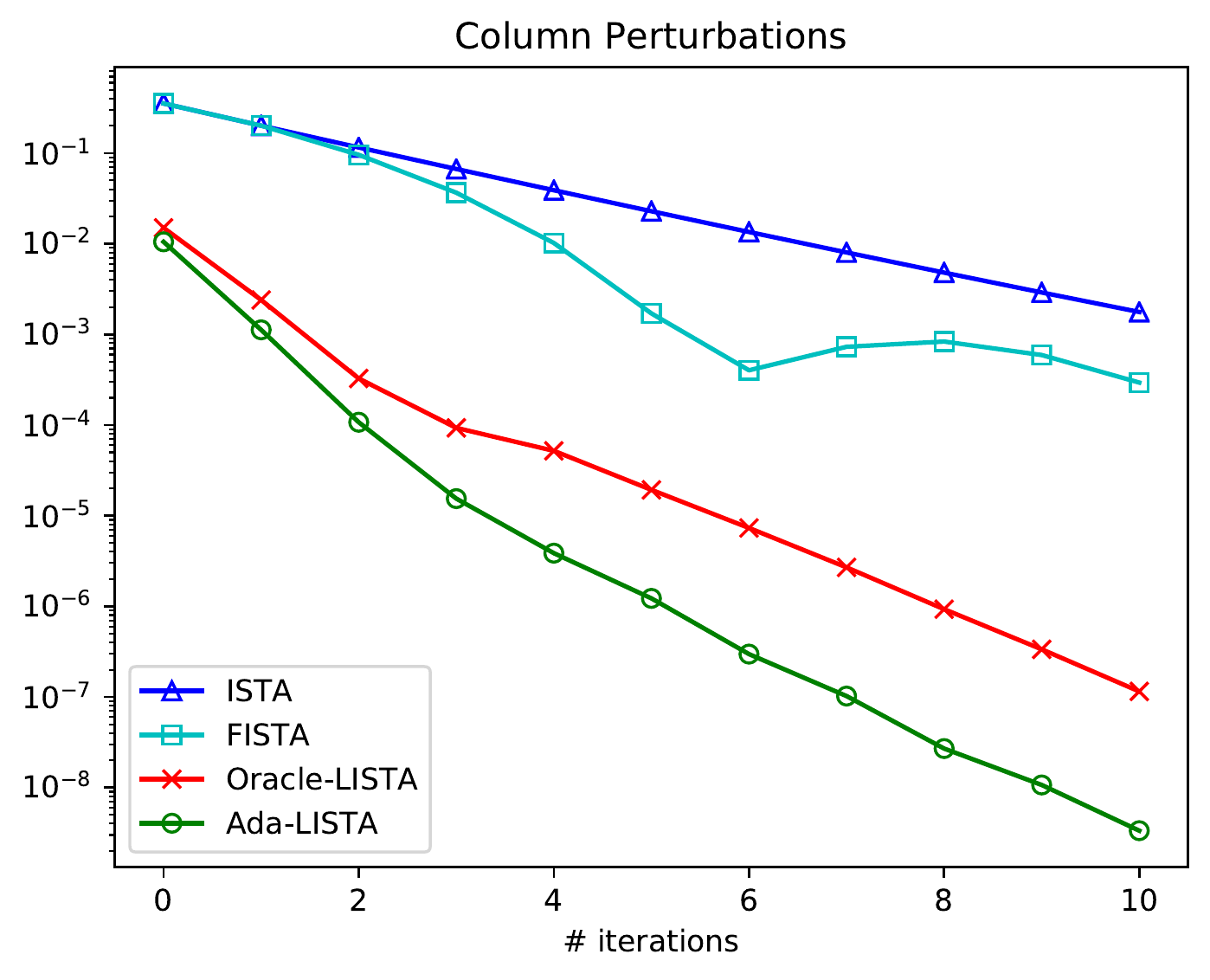}
        \caption{Signal SNR: $20$[dB].}
        \label{fig:results_column_noisy_snr_20}
    \end{subfigure}
    ~
    \begin{subfigure}[t]{0.3\linewidth}
        \centering
        \includegraphics[width=1.05\linewidth]{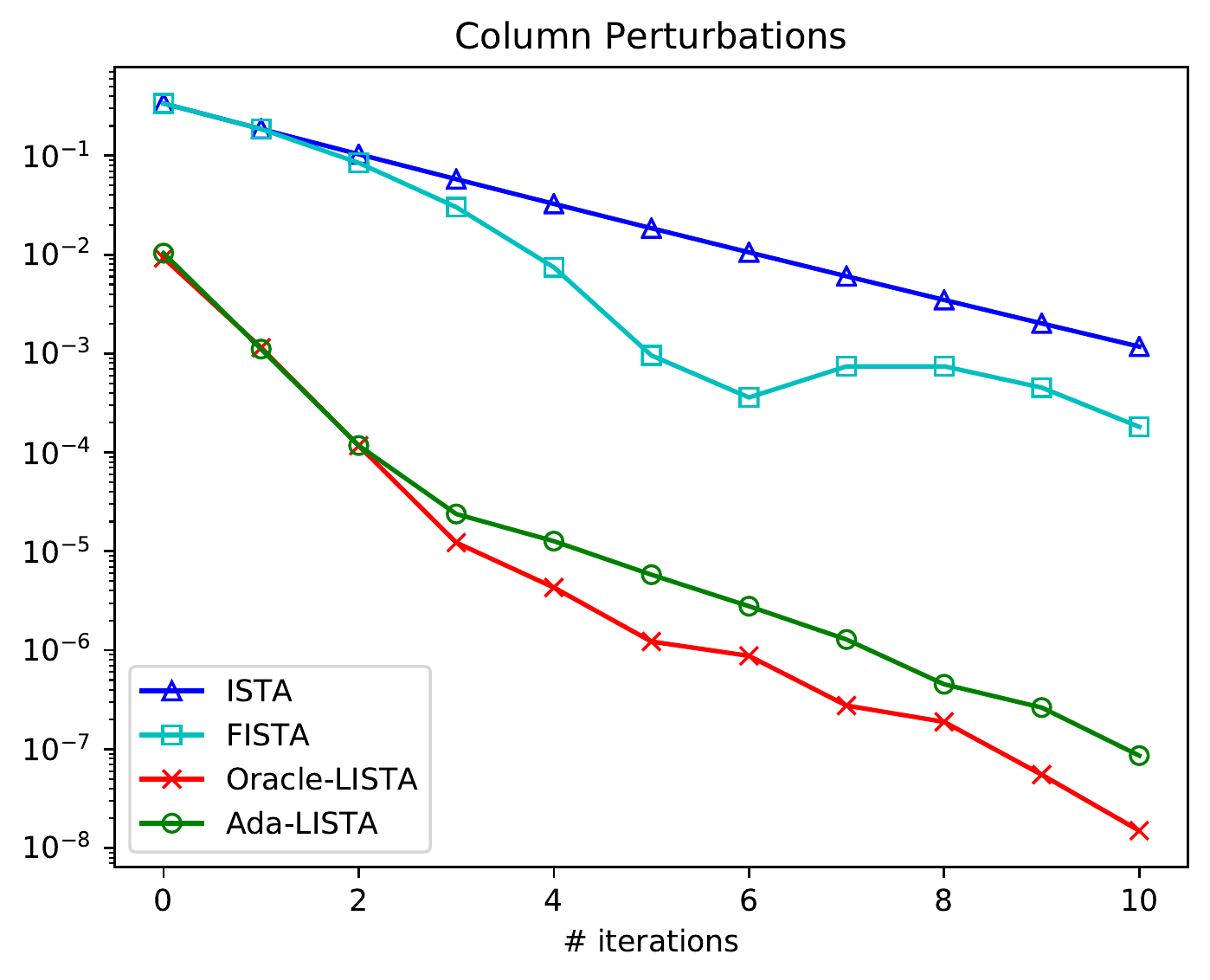}
        \caption{Signal SNR:  $30$[dB].}
        \label{fig:results_column_noisy_snr_30}
    \end{subfigure}
    \caption{MSE performance under column permutations and noisy inputs.}
    \label{fig:results_column_noisy}
\end{figure*}

\begin{figure*}[ht!]
    \centering
    \begin{subfigure}[t]{0.3\linewidth}
        \centering
        \includegraphics[width=1.05\linewidth]{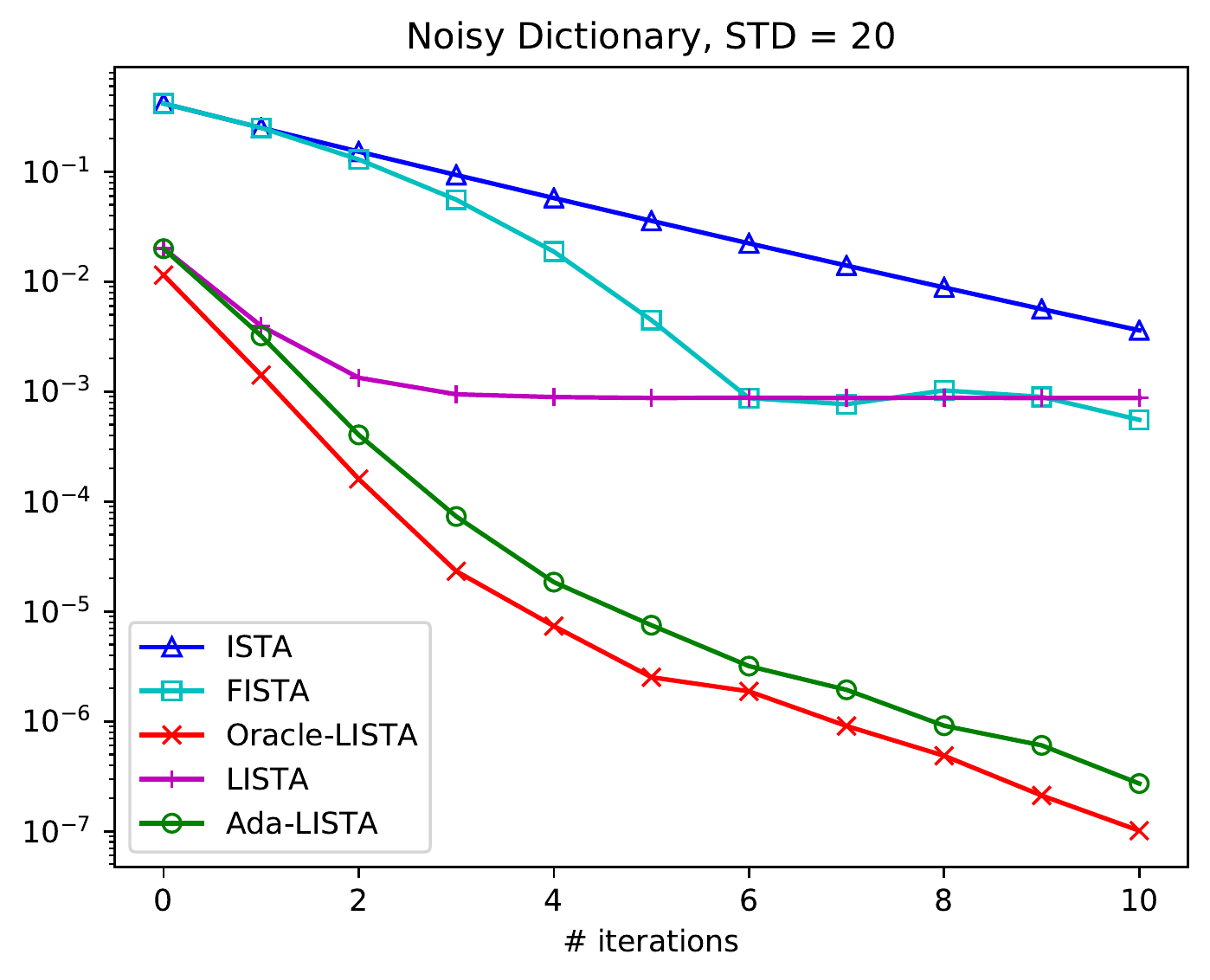}
        \caption{Signal SNR: $10$[dB].}
        \label{fig:results_noisy_dict_noisy_sig_snr_10}
    \end{subfigure}
    ~
    \begin{subfigure}[t]{0.3\linewidth}
        \centering
        \includegraphics[width=1.05\linewidth]{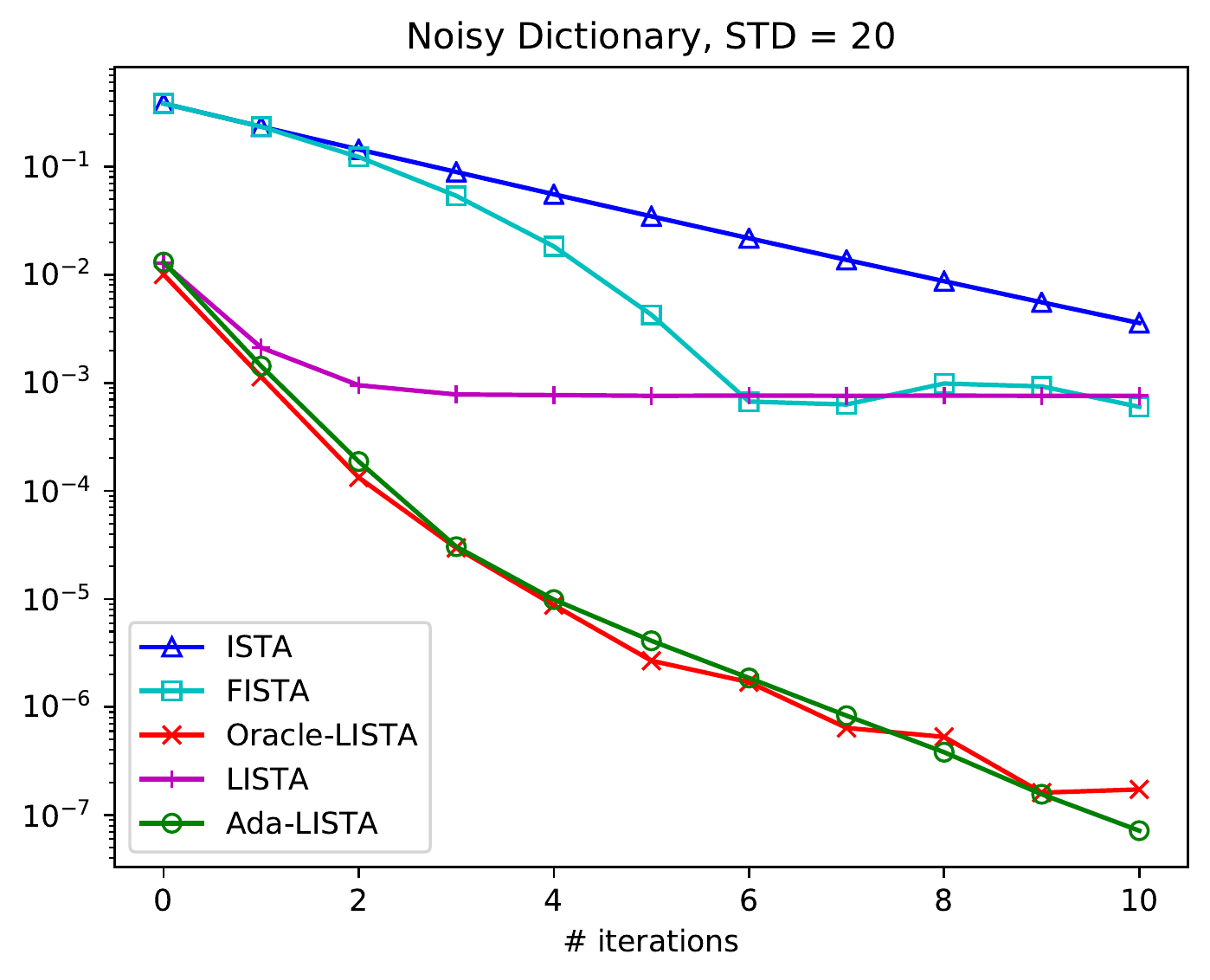}
        \caption{Signal SNR:  $20$[dB].}
        \label{fig:results_noisy_dict_noisy_sig_snr_20}
    \end{subfigure}
    ~
    \begin{subfigure}[t]{0.3\linewidth}
        \centering
        \includegraphics[width=1.05\linewidth]{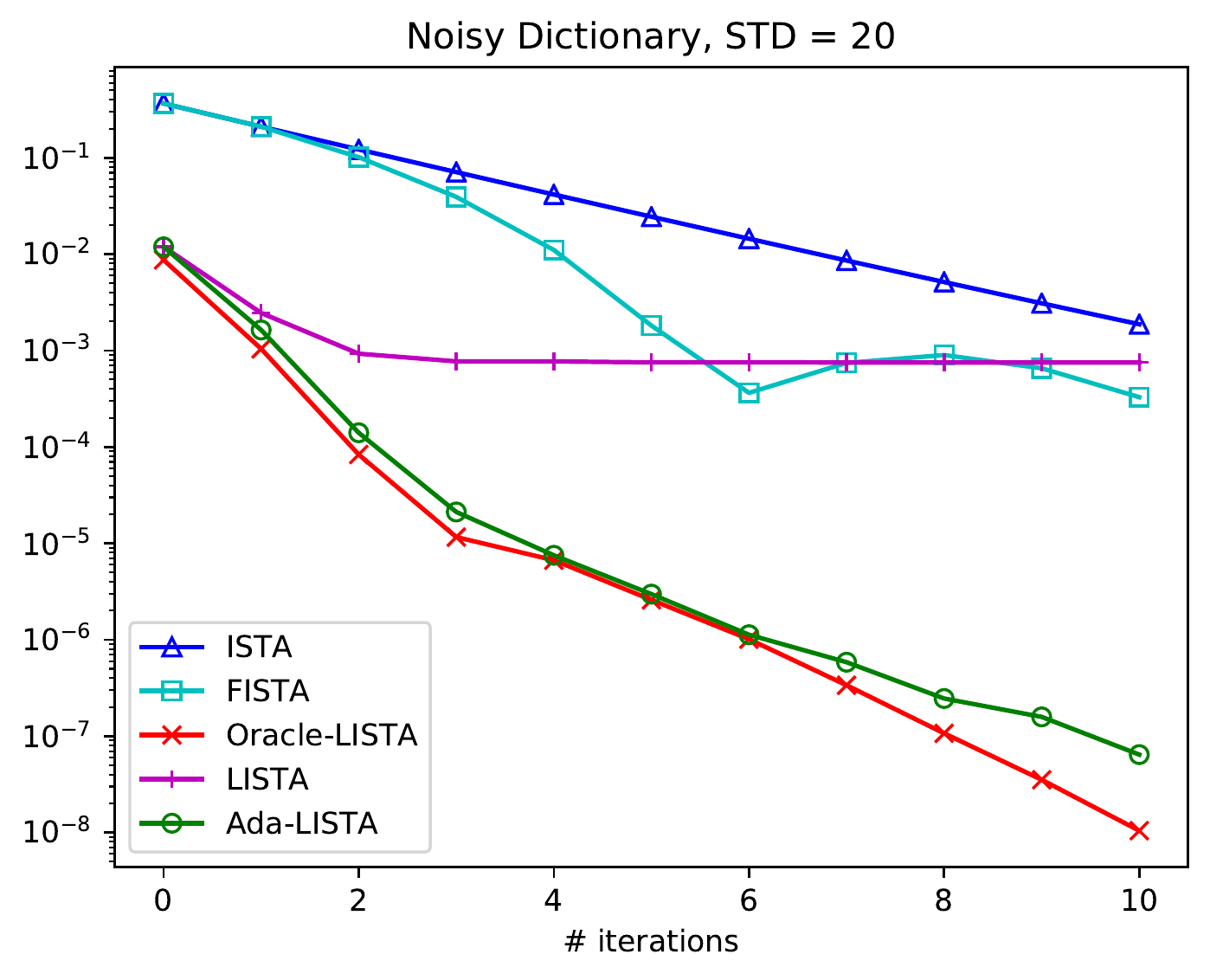}
        \caption{Signal SNR: $30$[dB].}
        \label{fig:results_noisy_dict_noisy_sig_snr_30}
    \end{subfigure}
    \caption{MSE performance for noisy dictionaries and noisy inputs.}
    \label{fig:results_noisy_dict_noisy_sig}
\end{figure*}

\begin{figure*}[ht!]
    \centering
    \begin{subfigure}[t]{0.3\linewidth}
        \centering
        \includegraphics[width=1.05\linewidth]{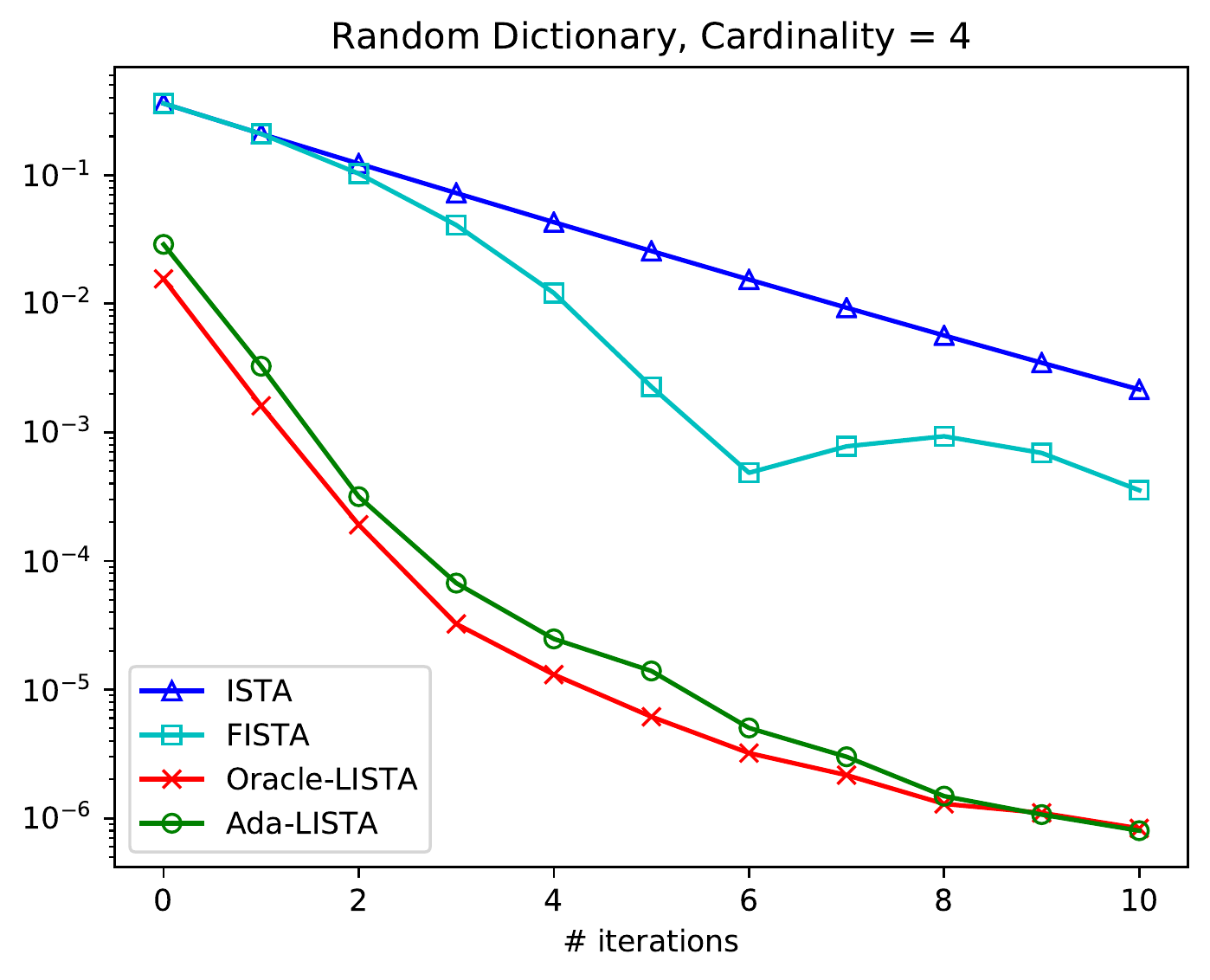}
        \caption{Signal SNR: $10$[dB].}
        \label{fig:results_random_dict_noisy_sig_snr_10}
    \end{subfigure}
    ~
    \begin{subfigure}[t]{0.3\linewidth}
        \centering
        \includegraphics[width=1.05\linewidth]{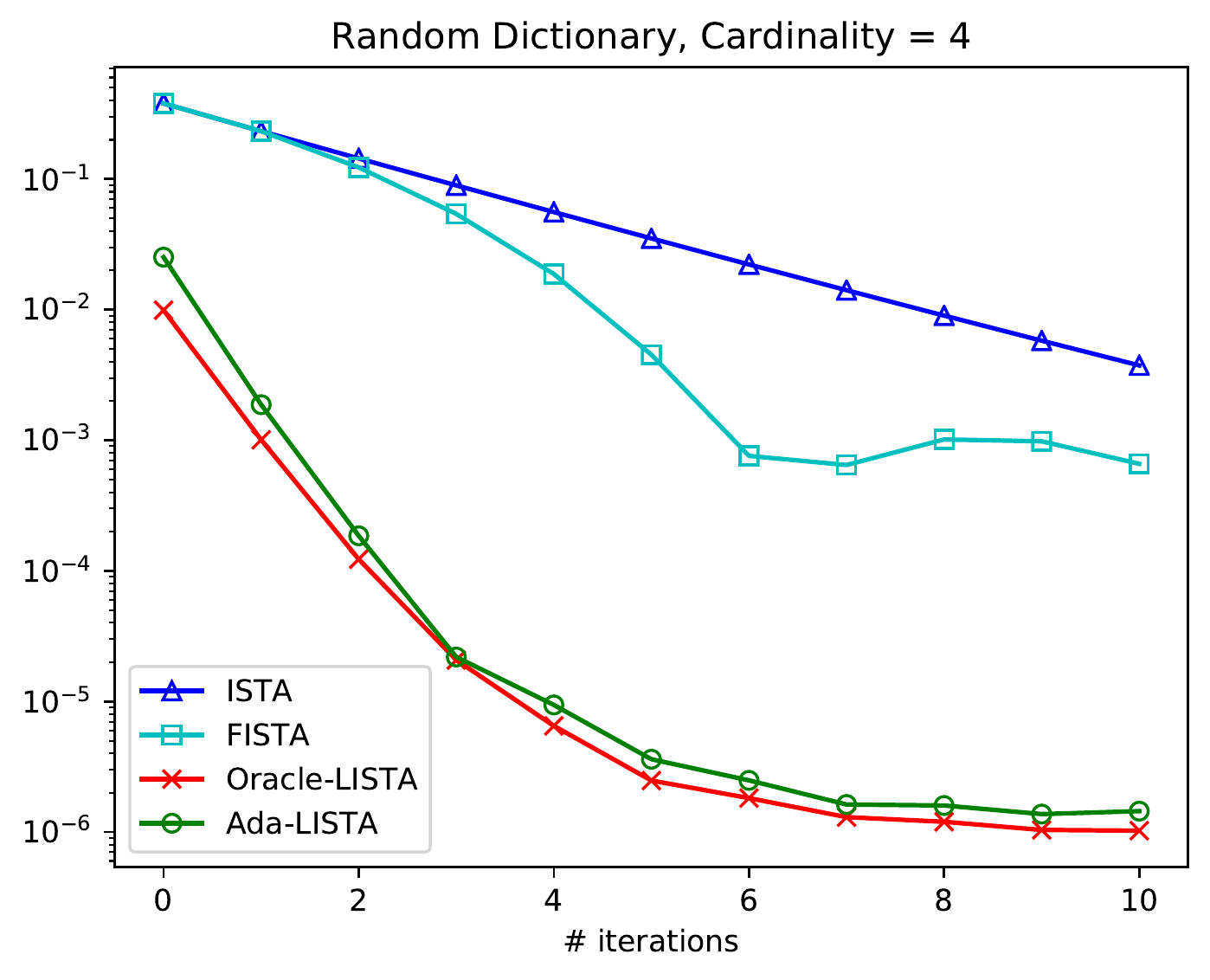}
        \caption{Signal SNR: $20$[dB].}
        \label{fig:results_random_dict_noisy_sig_snr_20}
    \end{subfigure}
    ~
    \begin{subfigure}[t]{0.3\linewidth}
        \centering
        \includegraphics[width=1.05\linewidth]{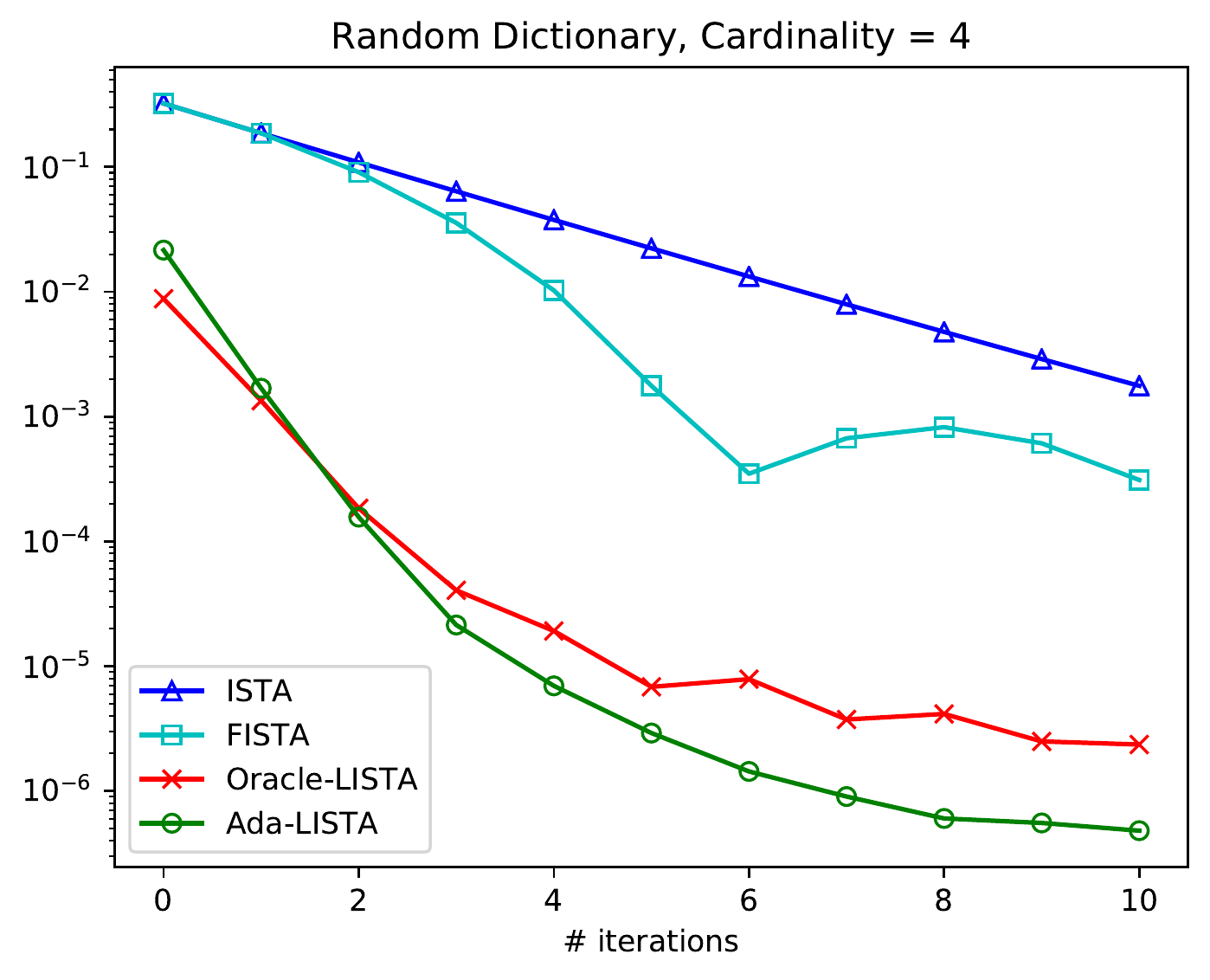}
        \caption{Signal SNR: $30$[dB].}
        \label{fig:results_random_dict_noisy_sig_snr_30}
    \end{subfigure}
    \caption{MSE performance under random dictionaries and noisy inputs.}
    \label{fig:results_random_dict_noisy_sig}
\end{figure*}

\section{Comparison to Robust-ALISTA}
\label{app:alista}

A similar concept of robustness to model noise is suggested in ``robust-ALISTA'' \cite{liu2018alista}, which models the clean dictionary $\rmD$ as having small perturbations of the form $\trmD = \rmD + \rmE$. This robustness is achieved in two stages, the first, computing a weight matrix $\trmW$ for each noisy model $\trmD$ by minimizing:
\begin{equation} \label{eq:alitsa_matrix}
    \trmW = \argmin_{\rmW} \norm{\rmW^T\trmD}_F^2, ~ \st ~ \rvw_i^T \trvd_i =1, ~ \forall i \in [1,m], 
\end{equation}
where $\rvw_i, \trvd_i$ are the $i$th columns of $\rmW$ and $\trmD$ respectively. Secondly, the matrix $\trmW$ is inserted into the ALISTA scheme:
\begin{equation} \label{eq:robust_alista}
    \rvx_{k+1} = \gS_{\theta_{k+1}} \left( \rvx_k - \gamma_{k+1} \trmW^T(\rmD\rvx_k - \rvy) \right),
\end{equation}
where the step sizes and thresholds $\{\gamma_k,\theta_k\}$ are learned parameters.
The advantage of this approach is the remarkably reduced number of trained parameters, $2K$. 
This method, however, suffers from several drawbacks. First, compared to Ada-LISTA, ALISTA is restricted to small model perturbations, and cannot handle more general scenarios, such as random dictionaries or even column perturbations. Second, in terms of computational complexity, robust-ALISTA has a complicated calculation of the analytic matrices during both training and inference (\Eqref{eq:alitsa_matrix}), a limitation that does not exist in our scheme.
Lastly, robust-ALISTA's training targets the original sparse representations that generated the signals. This makes ALISTA both impractical to real-world scenarios and restricted to sparse coding applications. Ada-LISTA, on the other hand, operates with accessible ISTA/FISTA solutions of \Eqref{eq:basis_pursuit_objective}, and thus can be used for any generic problem, solvable with ISTA (\Eqref{eq:convex_problem}), e.g., low-rank matrix models \cite{sprechmann2015learning}, acceleration of Eulerian fluid simulation \cite{tompson2017accelerating} and feature learning \cite{andrychowicz2016learning}.

\section{Image Inpainting Results}
\label{app:inpainting}

Figures \ref{fig:inpainting_app_1}, \ref{fig:inpainting_app_2} present the qualitative inpainting results on the rest of the \texttt{Set11} images, presented in subsection \ref{subsec:inpainting}.

\begin{figure*}[ht]
    \centering
    \inpaintings{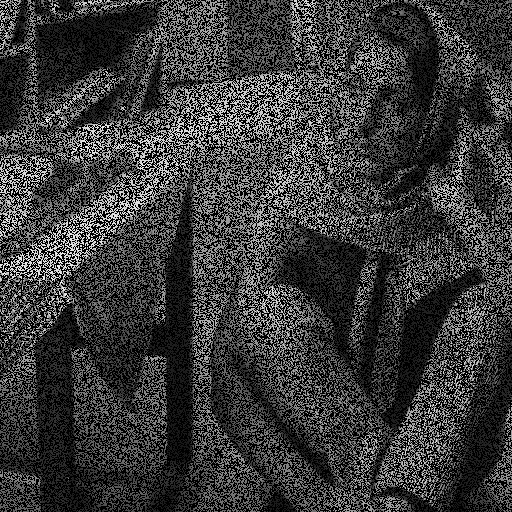} \inpaintings{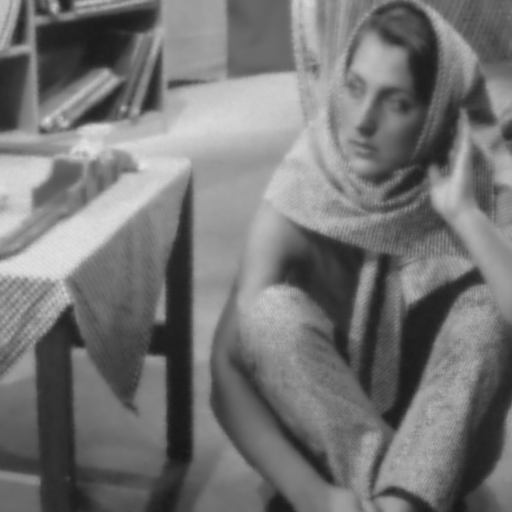}
    \inpaintings{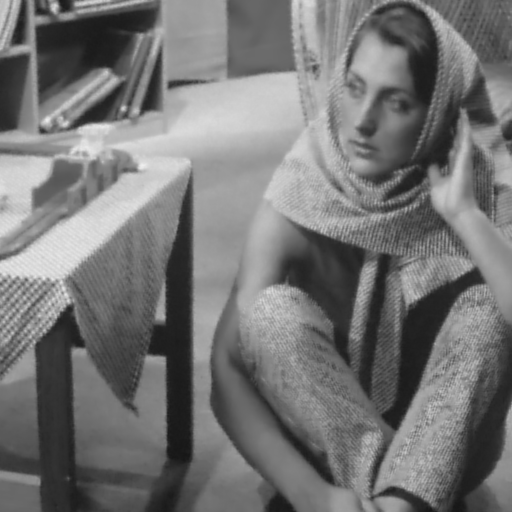} \inpaintings{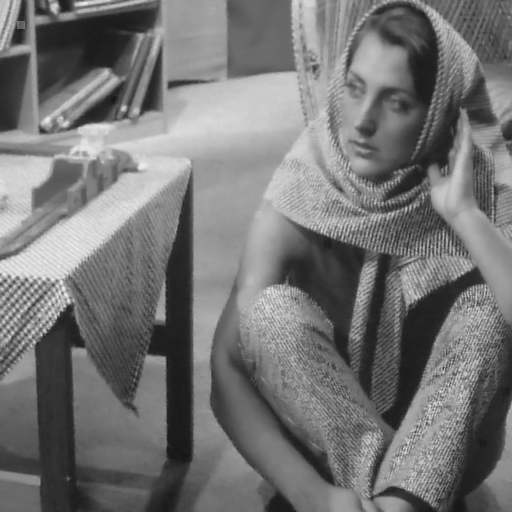} \\
    
    \inpaintings{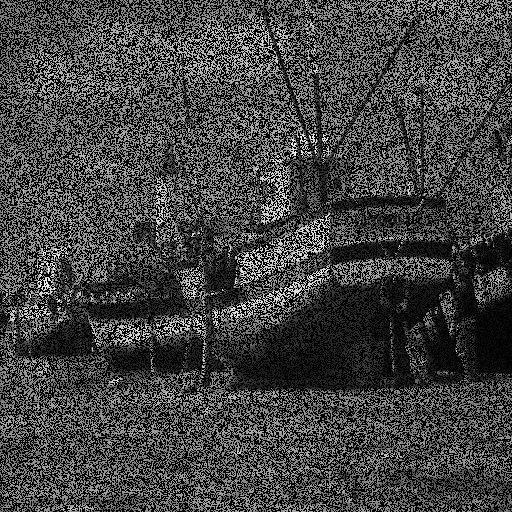} \inpaintings{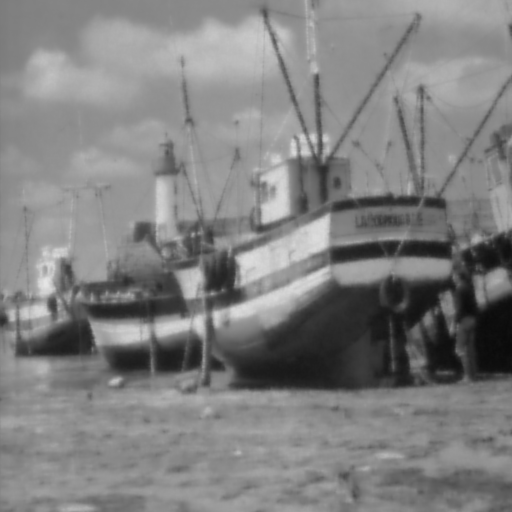}
    \inpaintings{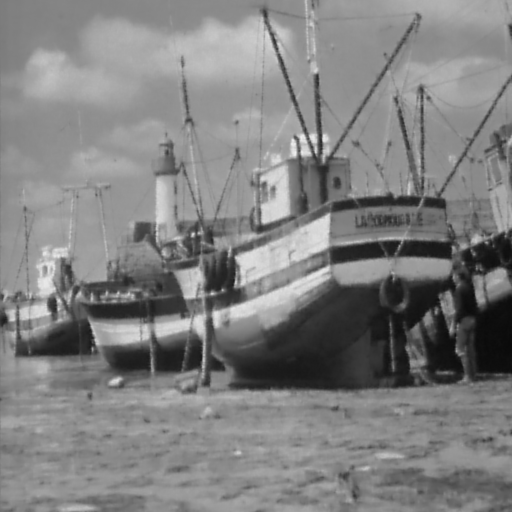} \inpaintings{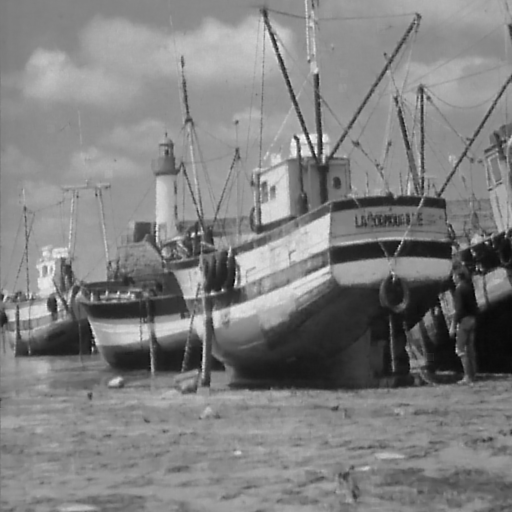} \\
    
    \inpaintings{Cameraman_corrupt} \inpaintings{Cameraman_ista}
    \inpaintings{Cameraman_fista} \inpaintings{Cameraman_adalista} \\
    
    \inpaintings{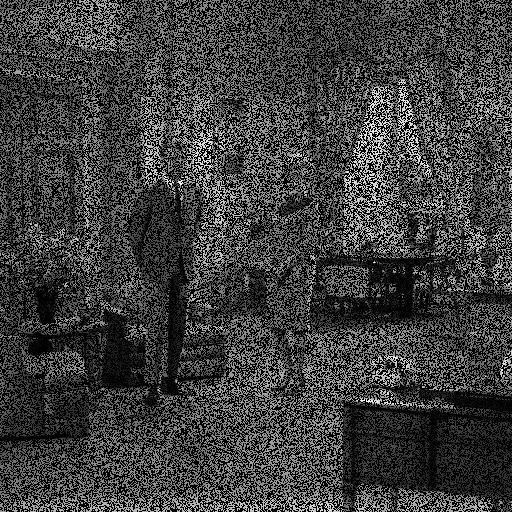} \inpaintings{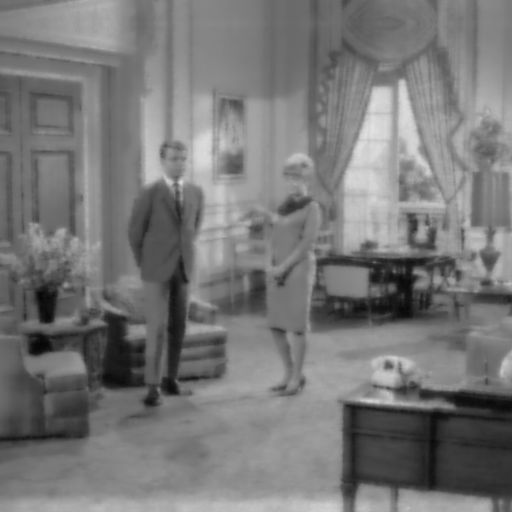}
    \inpaintings{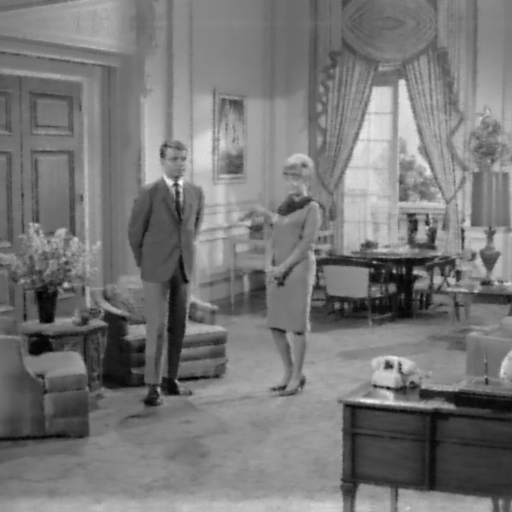} \inpaintings{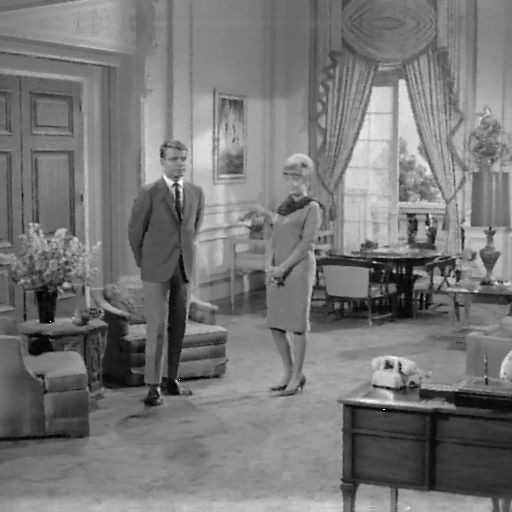} \\
    
    \inpaintings{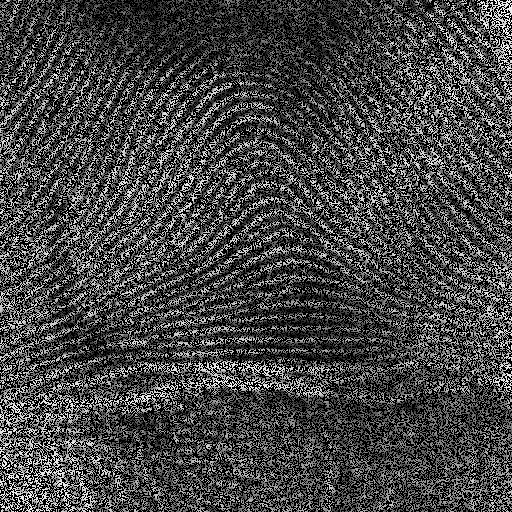} \inpaintings{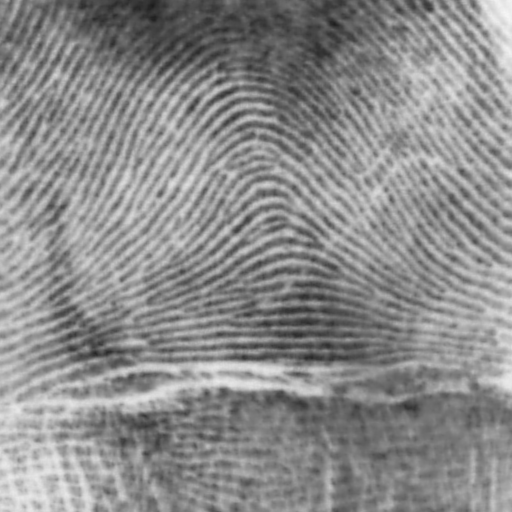}
    \inpaintings{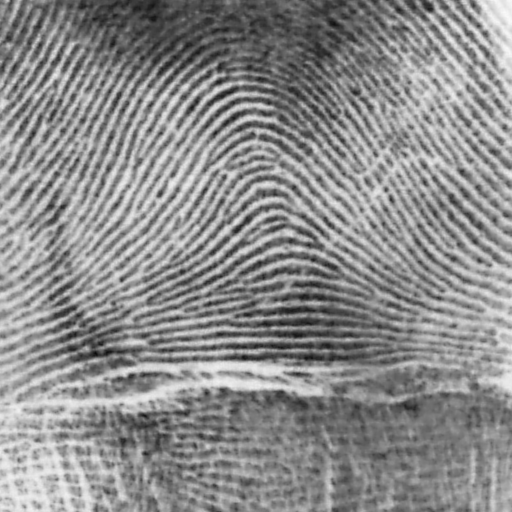} \inpaintings{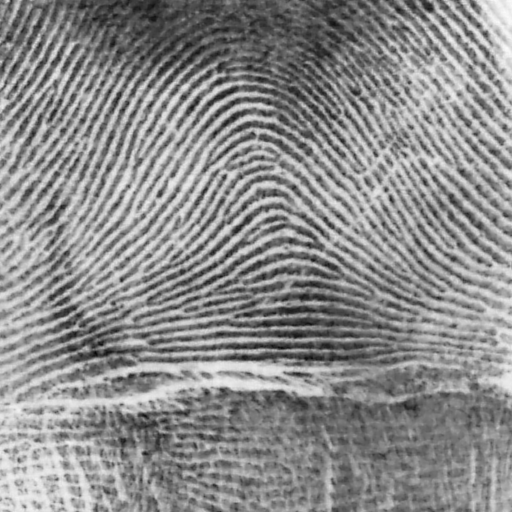} \\
    
    \inpaintings{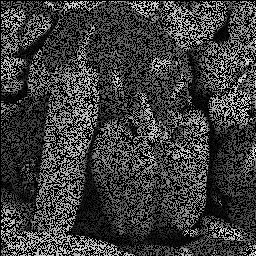} \inpaintings{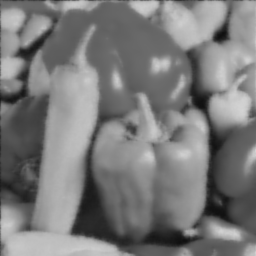}
    \inpaintings{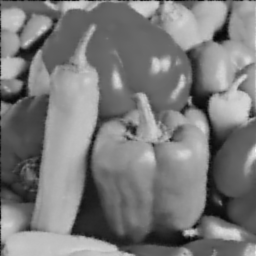} \inpaintings{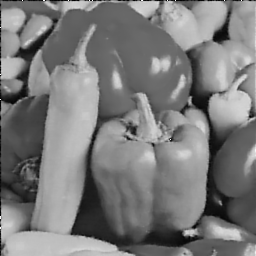} 
    
    \caption{Image inpainting with $50\%$ missing pixels. From left to right: corrupted image, ISTA, FISTA, and Ada-LISTA.}
    \label{fig:inpainting_app_1}
\end{figure*}

\begin{figure*}[ht]
    \centering
    
    \inpaintings{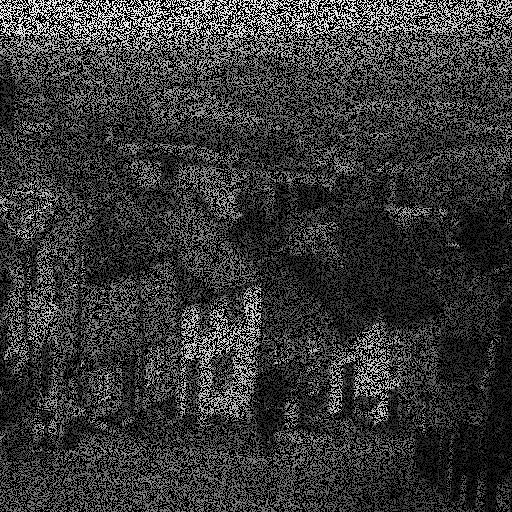} \inpaintings{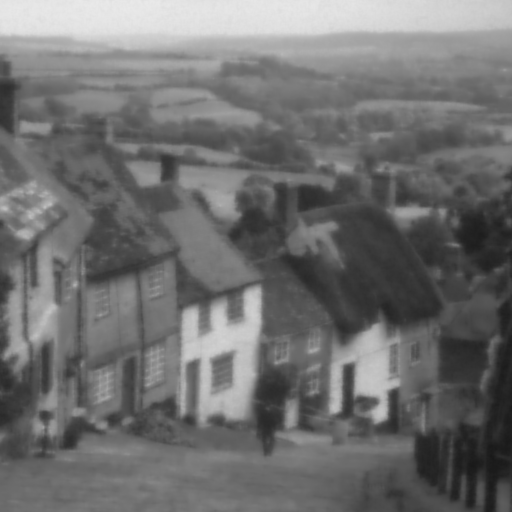} 
    \inpaintings{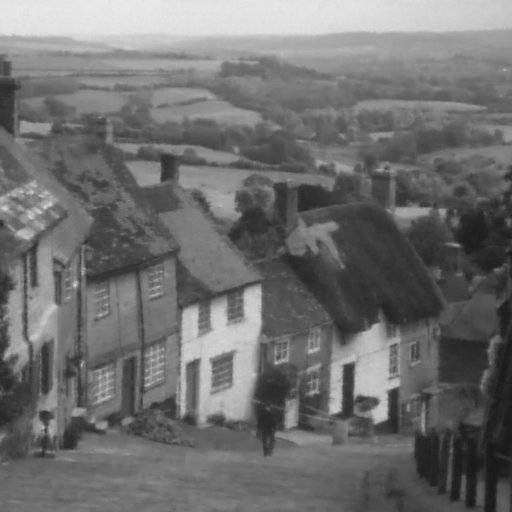} \inpaintings{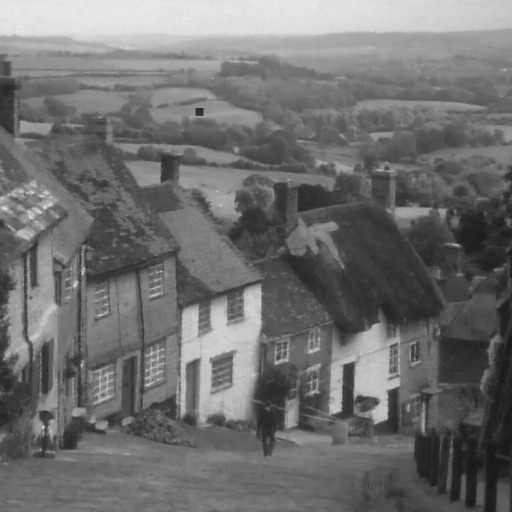} \\
    
    \inpaintings{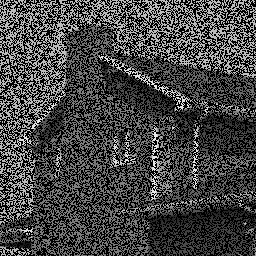} \inpaintings{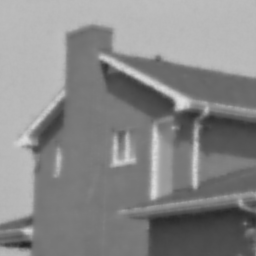} \inpaintings{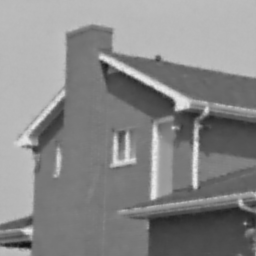} \inpaintings{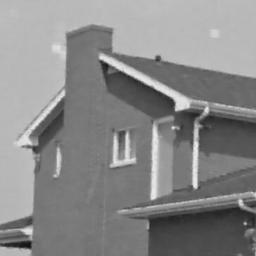} \\
    
    \inpaintings{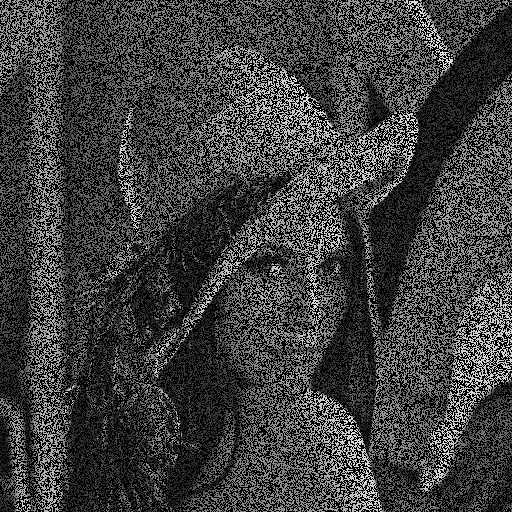} \inpaintings{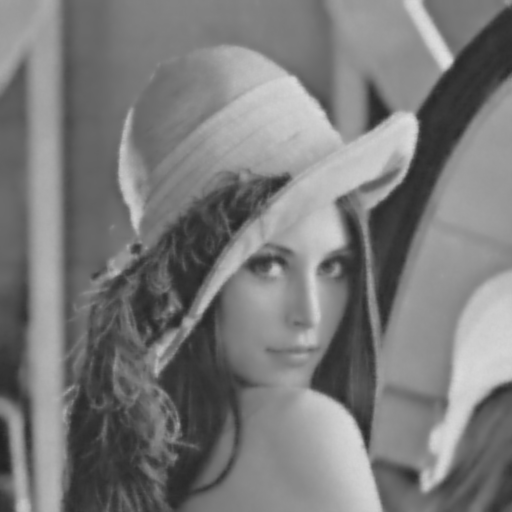}
    \inpaintings{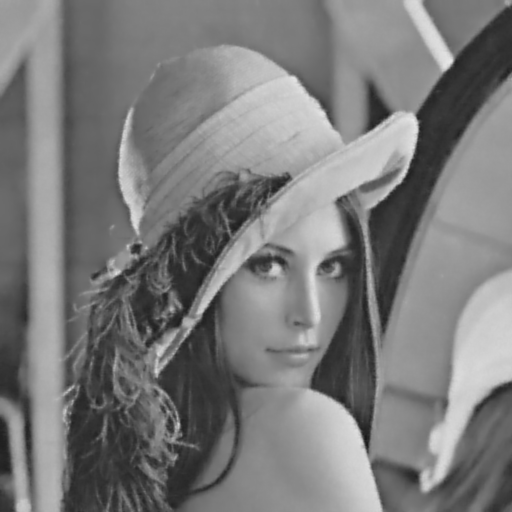} \inpaintings{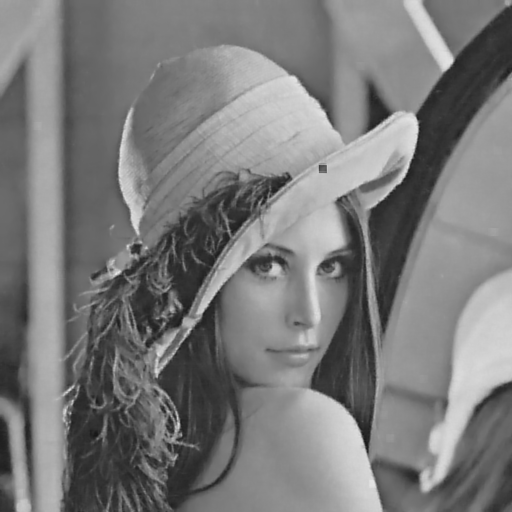} \\
    
    \inpaintings{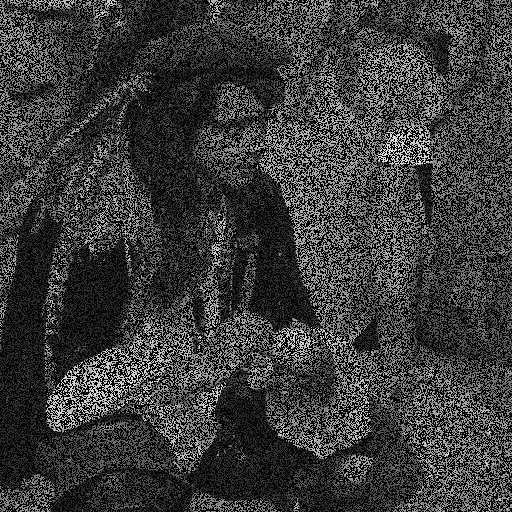} \inpaintings{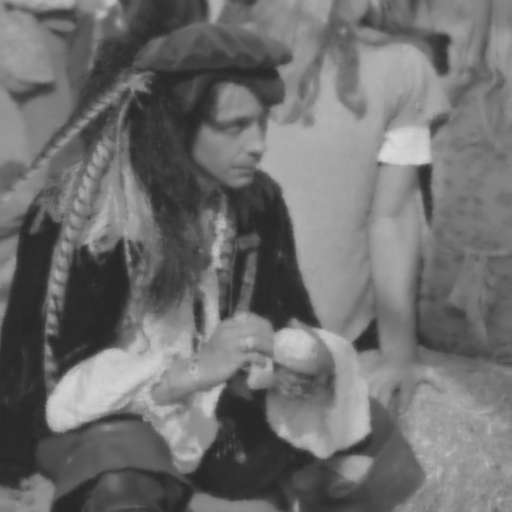}
    \inpaintings{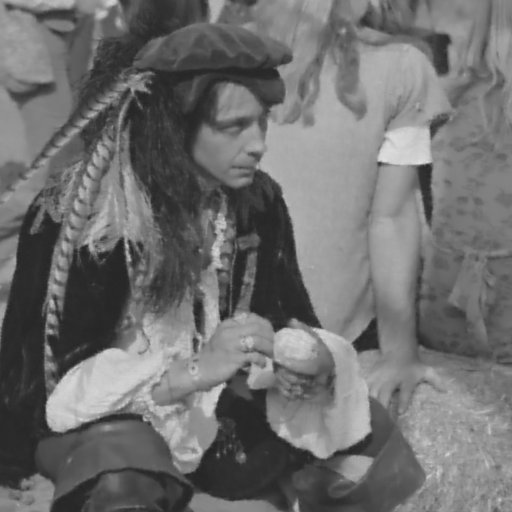} \inpaintings{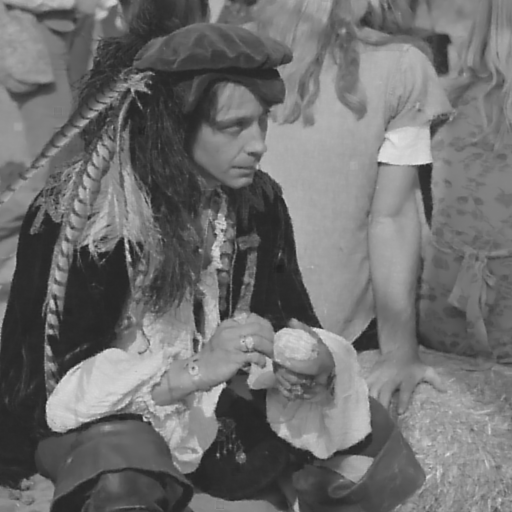} \\
    
    \inpaintings{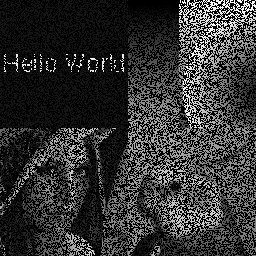} \inpaintings{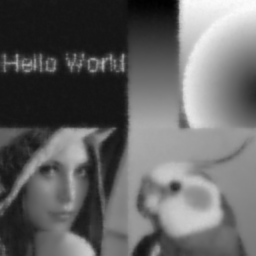}
    \inpaintings{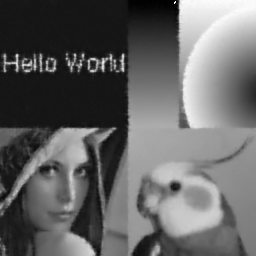} \inpaintings{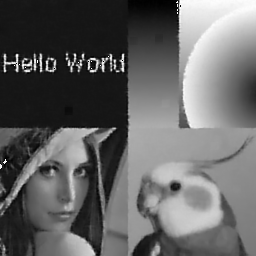} 
    
    \caption{Image inpainting with $50\%$ missing pixels. From left to right: corrupted image, ISTA, FISTA, and Ada-LISTA.}
    \label{fig:inpainting_app_2}
\end{figure*}

\end{document}